%% file: main.tex
\documentclass[11pt]{article}

\usepackage[utf8]{inputenc}
\usepackage[T1]{fontenc}
\usepackage{lmodern}
\usepackage[margin=1in]{geometry}
\usepackage{amsmath,amssymb,amsthm,mathtools}
\usepackage{booktabs,multirow,array}
\usepackage{graphicx}
\usepackage{subcaption}
\usepackage{xcolor}
\usepackage{algorithm}
\usepackage{algpseudocode}
\usepackage{tikz}
\usetikzlibrary{positioning,fit,calc,arrows.meta,backgrounds}
\usepackage[numbers,sort&compress]{natbib}
\usepackage[colorlinks=true,linkcolor=blue!60!black,citecolor=blue!60!black,urlcolor=blue!60!black]{hyperref}
\usepackage[capitalise,noabbrev]{cleveref}
\makeatletter
\providecommand*{\theHALG@line}{}
\renewcommand*{\theHALG@line}{\thealgorithm.\arabic{ALG@line}}
\makeatother
\usepackage{enumitem}
\usepackage{microtype}
\usepackage[section]{placeins}
\graphicspath{{figures/}}

\newtheorem{proposition}{Proposition}

\theoremstyle{definition}

\theoremstyle{remark}

\newcommand{\R}{\mathbb{R}}
\newcommand{\norm}[1]{\left\lVert #1 \right\rVert}
\newcommand{\ip}[2]{\left\langle #1, #2 \right\rangle}
\newcommand{\Tz}{T^{0}}
\newcommand{\Dempty}{\Delta_{\mathrm{empty}}}
\newcommand{\Ddonor}{\Delta_{\mathrm{donor}}}

\title{Delta N-gram: Forward-Only Continual Writing\\
into a Language Model's Native N-gram Memory}

\author{
  Hao Wu \\
  Illinois Institute of Technology \\
  \texttt{hwu71@hawk.illinoistech.edu}
}

\date{}

\begin{document}
\maketitle

\input{sections/00_abstract}
\input{sections/01_introduction}
\input{sections/02_related}
\input{sections/03_method}
\input{sections/04_analysis}
\input{sections/05_setup}
\input{sections/06_memory}
\input{sections/07_longchain}
\input{sections/08_rows}
\input{sections/09_discussion}
\input{sections/10_conclusion}

\bibliographystyle{plainnat}
\bibliography{refs}

\appendix
\input{sections/A_appendix}

\end{document}

%% file: sections/00_abstract.tex
\begin{abstract}
Our previous analysis found that a model's native hashed N-gram table is where continual updates disturb the base least, but its writes needed backpropagation and changed the pretrained table. We present Delta N-gram, a forward-only write path into that table. The pretrained table $T^0$ is never modified. Evidence is written into a sparse patch $\Delta$ on the rows it addresses: a 4.1M-parameter writer emits a target and a gate per (position, hash head) event, and a block-averaged delta rule keeps each row a convex combination of its old value and the gated targets. The writer is trained offline on raw text with the pretraining objective; at use time writing takes forward passes only, and clearing $\Delta$ restores the original model. A 0.6B prototype trained to answer questions keeps 16/16 facts after a second write; online gradient writes keep 2--4/16. We study two writers: a plain writer that encodes evidence with the $T^0$-only model, and a delta writer that reads its current memory while encoding. On Qwen3.8-Flash-Next, the plain writer's memory is document-specific: on 64 held-out news events it lowers continuation NLL by 0.068 nats/token against empty memory and by 0.024 against another document's memory, and 0.063 and 0.013 after 16 intervening articles. The delta writer's memory is more document-specific and keeps more of it under distraction (70\% vs 54\%). Over 35,000 writes on our previous seven-task chain both writers stay stable and the delta writer drifts less (layer-2 math drift 176\% vs 188\%); later accuracy changes by $-0.26$/$-0.52$pp ($-7.29$pp for gradient updates) and small-$k_i$ drift is 17.9/17.5\% (39.2\%). Acquisition is negative, and general-input drift reaches 32.8\% (20.75\%): format rows read by every example grow to 4--12 times their pretrained norm. Zeroing eight such rows removes over 98\% of layer-2 drift on math inputs, together with their NLL gain.
\end{abstract}

%% file: sections/01_introduction.tex
\section{Introduction}
\label{sec:intro}

A model that learns while it is used has to write new information somewhere, and every write is paid for by what it disturbs. In our previous work, \emph{Representation Drift}~\citep{wu2026drift}, we compared three places to write on the same 48-layer mixture-of-experts model: the rows of its native, pretrained hashed N-gram table~\citep{qwen38flashnext}, one continually trained LoRA adapter~\citep{hu2022lora}, and a fresh LoRA per task merged into the backbone. Updating only the N-gram table changed the base model's representations of general inputs least (big-$K$ drift 26\%/15\% on math/code, against 64--97\%/35--39\% for LoRA), and this drift did not grow with the number of tasks. We concluded that long-running online learning should write new knowledge only into the memory table and keep the backbone fixed.

That conclusion was reached with an expensive and irreversible write. Every table update in~\citep{wu2026drift} was an Adam step: a full backward pass through the network per write, optimizer state per row, labels to define the loss, and a table that, once changed, no longer contains its pretrained values. A system that writes continually during use needs the opposite properties. Writing should cost a forward pass, need no labels or task boundaries, and leave the pretrained knowledge recoverable.

This paper builds that write path. We keep the pretrained table $\Tz$ fixed and write into a sparse patch $\Delta$, so the model reads $\Tz+\Delta$. A small writer, trained offline, turns each block of evidence into writes. For every event, that is, every (position, hash head) pair, it emits a target $v$ and a gate $\beta$ from the frozen model's hidden states. Events that address the same row in a block are averaged and applied with a delta rule we call D1. D1 moves the row toward the gated targets and keeps it a convex combination of its old value and those targets. Seen from linear attention, this is a DeltaNet~\citep{schlag2021linear,yang2024deltanet} whose keys are one-hot hash addresses, whose state is never reset, and whose write targets come from a learned fast-weight programmer~\citep{schmidhuber1992fastweights}. The writer is trained with the pretraining objective on raw text. Each document is split into chunks; each chunk is first predicted from the memory written so far and then written, and the loss of later chunks backpropagates through the whole chain of writes. At use time the same writer and the same D1 run forward only.

The writer must read the evidence before it writes, and we study two ways of doing so. The \emph{plain writer} encodes evidence with the model reading only $\Tz$. Its targets are then anchored to the evidence, rows stay separable, and all blocks of a document can be encoded in one batch. The \emph{delta writer} encodes evidence with the model reading its current memory, so what it writes can depend on what it has already stored. That is what long-term continual learning ultimately needs, for instance to connect a new report to an entity written earlier. Its price is sequential writing and a feedback loop through the memory. We analyse that loop and train both writers on the same data.

\paragraph{What the design gives by construction.}
\begin{itemize}[leftmargin=*,itemsep=2pt]
  \item \textbf{No new memory module.} Memory is the model's own N-gram table, addressed by pretrained hashes of the last two or three tokens and read by the model's native injection at layer~2.
  \item \textbf{Pretrained parameters are untouched.} Rows that were never written read exactly $\Tz$; clearing $\Delta$ returns the original model, and each user's $\Delta$ is a separable set of sparse rows.
  \item \textbf{Forward-only writes.} At use time there is no backpropagation, no optimizer state and no label.
  \item \textbf{Row-level attribution.} Every change lives on a known row, so drift can be traced to, and intervened on, individual rows.
\end{itemize}

\paragraph{Findings.}
We train both writers (4.1M parameters each) on Qwen3.8-Flash-Next (512 experts, 48 layers) for 100M target tokens. We evaluate them on held-out news events, on the seven-task chain of~\citep{wu2026drift}, and on LoCoMo~\citep{maharana2024locomo}.
\begin{enumerate}[leftmargin=*,itemsep=2pt]
  \item \textbf{Trained to answer, forward writes bind and keep facts, but only in the trained form.} On a 0.6B prototype trained with a question-answering objective, the writer keeps 16 of 16 facts after a second write, where online gradient writes into the same table keep 2--4. With a paraphrased question it stays at the empty-memory level (\cref{sec:memory-proto}).
  \item \textbf{Raw-text training yields document-specific memory.} On 64 late-2025 news events, reading a document's own memory beats reading another document's, which beats empty memory, at every training stage. At 100M tokens the two gains are 0.068 and 0.024 nats/token, and 0.063 and 0.013 after 16 intervening articles (\cref{sec:memory}).
  \item \textbf{Reading memory while encoding changes how the writer writes.} The delta writer's memory is more specific to its document ($\Ddonor$ 0.032 against 0.023 without distractors) and keeps more of that specificity after 16 distractors (70\% against 54\%); its total gain over empty memory on these news events is smaller. A weight $\times$ encoding factorial places this change in the learned weights (\cref{sec:memory-encoding}). On the seven-task chain the delta writer drifts less than the plain writer: math drift is 176\% against 188\% at layer~2 and 40.2\% against 41.8\% at the head. It also answers slightly more (129 against 125 of 448 test items, and 112 against 103 written examples) (\cref{sec:longchain}).
  \item \textbf{Long write chains are stable.} Both writers, including the delta writer whose encoding feeds back through its own memory, made 35,000 consecutive writes on two hardware stacks without numerical failure (\cref{sec:longchain}).
  \item \textbf{What was stored changes little afterwards, but little is stored.} After each task is written, learned-task accuracy changes by $-0.26$pp (plain) and $-0.52$pp (delta) over the rest of the chain, against $-7.29$pp for the Adam table updates of~\citep{wu2026drift}. Small-$k_i$ drift at the head input is 17.9\% and 17.5\%, against 39.2\%. Net acquisition on these tasks, which the writer never saw, is negative ($-12/448$), and recall of written examples is 23--25\% (\cref{sec:longchain}).
  \item \textbf{Drift concentrates on a few hot rows and can be located.} General-input big-$K$ drift is 32.8\%, 1.6 times that of~\citep{wu2026drift}. It enters at the injection layer through eight format bigram rows that every example writes, whose patches reach 4.4--11.7 times $\norm{\Tz}$. Zeroing them removes over 98\% of the layer-2 drift on math inputs and also removes their NLL gain (\cref{sec:rows}).
\end{enumerate}

\paragraph{Analysis.} \Cref{sec:analysis} gives the relations behind these measurements: the bound of old-task drift by prefix overlap times row displacement, inherited from~\citep{wu2026drift}; the boundedness of D1; row separability, under which zeroing a set of rows after the fact equals never committing them; and the fixed-point dynamics of a writer that encodes evidence through its own memory, which we then test with a weight $\times$ encoding factorial.

\paragraph{Position.} The two properties that a forward write path into the native table had to show, that writes carry information forward and that long chains of them stay stable without touching $\Tz$, hold. The two that remain, using what was written at a different position and bounding the absolute drift of hot rows, are where we take the method next (\cref{sec:discussion}). Code is available at \url{https://github.com/af-74413592/delta-ngram}.

%% file: sections/02_related.tex
\section{Related Work}
\label{sec:related}

\begin{table}[t]
  \centering
  \footnotesize
  \caption{Ways to learn while in use. ``Pretrained parameters'' refers to the weights and tables the model shipped with.}
  \label{tab:related}
  \setlength{\tabcolsep}{3pt}
  \begin{tabular}{>{\raggedright\arraybackslash}p{1.9cm}*{4}{>{\raggedright\arraybackslash}p{2.05cm}}>{\raggedright\arraybackslash}p{2.3cm}}
    \toprule
    & In-context writes (linear attn., GDN, KDA) & Online LoRA & Gradient table edits~\citep{wu2026drift} &Persistent memory block (MemT) & \textbf{Delta N-gram} \\
    \midrule
    Written state & sequence state & low-rank backbone increment & native N-gram table & added memory block (dense GDN-2 state) & \textbf{sparse patch on the native N-gram table} \\
    Write signal & forward projection & backprop & backprop (Adam) & forward (GDN-2 update) & \textbf{forward writer + D1} \\
    Lifetime & cleared at sequence end & permanent & permanent & permanent & \textbf{permanent (optional decay)} \\
    Pretrained parameters & unchanged & frozen; adapter separate & table overwritten & unchanged; module added & \textbf{unchanged; $\Tz$ fixed} \\
    Use-time cost & near zero & full backward per write; optimizer state per user & same & forward & \textbf{one $\Tz$-only forward + 4M-param writer} \\
    Undo / per user & n/a & yes (adapter per user) & needs a copy of the table & yes (state per user) & \textbf{clear $\Delta$; sparse rows per user} \\
    Rows touched per write & whole state & all adapter parameters & rows read (sparse) & block state (dense, routed) & \textbf{rows read (sparse)} \\
    Addressing & projected keys & none & hashed n-grams & learned keys & \textbf{hashed n-grams, pretrained} \\
    Extra training & with the backbone & none & none & 0.45--3.6B params, 34B tokens & \textbf{4M params, 100M tokens} \\
    \bottomrule
  \end{tabular}
\end{table}

\paragraph{Where to write.} Our previous work, \emph{Representation Drift}~\citep{wu2026drift}, compared update channels on the model we use and found that writing only into the native N-gram table perturbs general inputs least; its costs were interference among stored tasks and weak acquisition. This paper keeps that channel and replaces the write: forward instead of gradient, patch instead of in-place change. Forgetting under sequential fine-tuning~\citep{mccloskey1989catastrophic,kirkpatrick2017ewc,luo2023empirical,biderman2024lora} and subspace constraints on successive updates~\citep{farajtabar2020ogd,wang2023olora,liang2024inflora} motivate the comparison but act on dense parameters, not on addressed rows.

\paragraph{Delta rules and fast weights.} Linear attention can be read as a fast-weight memory written by a slow network~\citep{schmidhuber1992fastweights,schlag2021linear}. DeltaNet replaces additive writes with error-correcting delta writes~\citep{schlag2021linear,yang2024deltanet}. Gated DeltaNet adds a scalar decay~\citep{yang2025gdn}, and KDA makes the decay channel-wise~\citep{kimi2025linear}. Test-time training generalises the inner update to a learned objective~\citep{sun2024ttt}. Qwen3.8 itself contains such layers. These lines address interference among keys in a finite state that is cleared at the end of each sequence, which is why they need decay. Our keys are one-hot hash addresses over hundreds of millions of rows, the state persists across requests, and the writer is added to a frozen pretrained model rather than trained with it. Decay in our setting pulls toward $\Tz$ rather than toward zero (\cref{sec:method-state}).

\paragraph{Sparse fast-weight memories.} Sparse Delta Memory~\citep{sdm2026} enlarges the state of a gated linear RNN with an explicit memory whose slots are selected by product keys and written with a gated delta rule. Fast-weight product-key memory~\citep{fwpkm2026} updates the values of a product-key memory by chunk-level gradient descent on a reconstruction loss, at training and at inference time, and can carry them across adjacent segments. Both are layers trained together with the backbone, over learned keys. Ours adds a writer to a frozen pretrained model, writes into a table the model was pretrained to read, and uses no gradients at use time.

\paragraph{Editing memory tables.} Memory and product-key layers provide large sparsely addressed stores~\citep{lample2019pkm,berges2024memory}. Sparse memory finetuning updates only rows selected by their access statistics and forgets less than LoRA or full fine-tuning~\citep{lin2025sparsememory,smf2026}. EngramEdit writes facts into the N-gram embeddings of LongCat-Flash-Lite with a closed-form ridge solution and a penalty on reused rows~\citep{engramedit2026}. ENGRAFT optimises the Engram rows of Qwen3.8-Flash-Next by gradient descent with a KL constraint and keeps the result as a removable overlay~\citep{engraft2026}. All of these compute the write from labelled facts or tasks, by backpropagation or a solve. Ours writes from raw evidence in a forward pass, with no labels, into a separable patch. Knowledge editing in FFN weights~\citep{meng2022rome,meng2023memit,fang2025alphaedit} shares the key--value view we use for drift but edits dense projections.

\paragraph{Persistent memory for a frozen backbone.} The closest work is MemT-1.0~\citep{memt2026}. It also freezes the backbone, writes a per-user memory forward-only and trains the memory module offline. Its memory is a set of added blocks with dense GDN-2 state, routed over storage experts, updated token by token and committed per turn. It trains 0.45--3.6B parameters on 34B tokens: continued pretraining, then post-training that distils from the frozen backbone given the full history. Post-training adds about 10.69 points over the continued-pretraining checkpoint, the largest effect in its ablations. Our memory is the model's own N-gram table, read through its native injection, with a 4M-parameter writer trained on 100M raw-text tokens. Because the state is a sparse patch on known rows, its drift can be traced to individual rows (\cref{sec:rows}). MemT's results on long-conversation benchmarks~\citep{maharana2024locomo} show what a trained cross-position objective buys. Our LoCoMo failure (\cref{sec:longchain-locomo}) is consistent with that, and its full-context teacher is the first change we would borrow (\cref{sec:discussion}).

%% file: sections/03_method.tex
\section{Method}
\label{sec:method}

\input{sections/fig_overview}

\Cref{fig:overview} summarises the method. The backbone, the native read-out and the pretrained table $\Tz$ are frozen. The only trained component is a small writer, and it is trained offline; at use time the model writes into a sparse patch $\Delta$ with forward passes only.

\subsection{Native N-gram memory}
\label{sec:method-native}

A hashed N-gram table $T\in\R^{N\times d}$ is read at every position $p$ through $H$ hash heads. Head $h$ maps the last two or three tokens ending at $p$ to a row address $j_h(p)$ that depends only on token ids and is fixed at pretraining. We write the table as
\begin{equation}
  T=\Tz+\Delta ,
  \label{eq:patch}
\end{equation}
where $\Tz$ is the pretrained table and $\Delta$ is a patch that is nonzero only on rows that have been written. Every row with $\Delta_j=0$ reads exactly its pretrained value.

On Qwen3.8-Flash-Next~\citep{qwen38flashnext}, $H=16$ (bigram and trigram, eight heads each) and each row has $d=160$ dimensions. The read-out concatenates the 16 rows and injects the result into the residual stream before the attention of layer~2:
\begin{gather}
  e_q=\big[T_{j_1(q)};\dots;T_{j_{16}(q)}\big]\in\R^{2560},\qquad
  v_q=W_ve_q,\qquad k_{q,c}=W_k^{(c)}e_q, \label{eq:read1}\\
  g_{q,c}=\sigma\!\Big(\phi\big(\ip{\mathrm N(k_{q,c})}{\mathrm N(z_{q,c})}/\sqrt{D}\big)\Big),\qquad \phi(x)=\mathrm{sign}(x)\sqrt{\max(|x|,10^{-6})}, \label{eq:read2}\\
  b_q=\operatorname{flat}_{c=1}^{4}\big(g_{q,c}v_q\big),\qquad
  u_q=b_q+\operatorname{SiLU}\!\Big(\operatorname{DWConv}\big(\mathrm N_{\mathrm{conv}}(b)\big)_q\Big)\in\R^{10240}. \label{eq:read3}
\end{gather}
Here $c$ indexes the four residual streams, $z_{q,c}$ is the hidden state after layer~1, $\mathrm N$ is RMSNorm, and the depthwise convolution has four taps with dilation~3. We use this read-out unchanged. The value path is linear in $e_q$, so a patch acts through its magnitude; the key and convolution paths are normalised, so there it acts through its direction.

Because every weight after the injection is frozen, a change of the table reaches the backbone only as a change of its input. In the key--value notation of~\citep{wu2026drift,fang2025alphaedit}, where a layer computes $Y=KW$ and $Y_t-Y_i=K_i(W_t-W_i)+(K_t-K_i)W_t$, only the second term survives: for the backbone, writing the table is a pure perturbation of $K$.

\subsection{From DeltaNet to a row-wise delta rule}
\label{sec:method-d1}

DeltaNet~\citep{schlag2021linear,yang2024deltanet} reads $o=Sq$ and writes $S\leftarrow S-\beta(Sk-v)k^\top$. We obtain the write rule for the table in four steps.
\begin{enumerate}[leftmargin=*,itemsep=2pt]
  \item \textbf{One-hot keys.} Treat each row of the table as a column of $S$ and set $k=e_j$. Reading $Se_j$ is a table lookup, and $kk^\top$ is nonzero only at $j$, so a write touches only row $j$. Distinct rows do not interfere except through hash collisions, and $\norm{k}=1$ holds by construction.
  \item \textbf{Write the patch.} The update acts on $\Delta_j$; $\Tz$ is never written.
  \item \textbf{Average within a block.} Evidence arrives in blocks. Every pair (position $p$, head $h$) in a block is a \emph{write event} $a$ addressed to row $j(a)=j_h(p)$. The $n_j$ events addressed to row $j$ all see the same pre-block snapshot of the state, and their effects are averaged.
  \item \textbf{Targets and step sizes come from a writer.} Where DeltaNet obtains $v$ and $\beta$ by forward projections, here a learned writer provides them (\cref{sec:method-writer}).
\end{enumerate}
Together these give the rule we call D1. For every row touched by the block,
\begin{equation}
  A_j=\frac1{n_j}\sum_{a\to j}\beta_a v_a,\qquad
  w_j=\frac1{n_j}\sum_{a\to j}\beta_a\in[0,1],\qquad
  \boxed{\;\Delta_j^{+}=(1-w_j)\,\Delta_j+A_j,\qquad c_j^{+}=c_j+1\;}
  \label{eq:d1}
\end{equation}
The count $n_j$ includes events with $\beta_a=0$; rows not touched by the block are left exactly unchanged; $c_j$ is the number of blocks that have written row $j$. Equivalently, $\Delta_j^{+}-\Delta_j=\frac1{n_j}\sum_a\beta_a(v_a-\Delta_j)$: each event moves the row in proportion to its residual $v_a-\Delta_j$, as in the delta rule. \Cref{sec:analysis-d1} shows that averaging, rather than summing, is what keeps the rule bounded.

\subsection{The writer}
\label{sec:method-writer}

When a block of evidence $x_t$ arrives, the writer proceeds in three steps.

\paragraph{Encoding.} The frozen model runs over the block, and we take the input of the LM head, normalised by a parameter-free RMSNorm (Qwen3.8 has no separate final norm, and the per-token scale of this input varies widely):
\begin{equation}
  \hat H_t=\mathrm N\big(E(x_t;\,\Tz+\lambda\Delta_t)\big),\qquad \lambda\in\{0,1\},
  \label{eq:encode}
\end{equation}
with gradients stopped at $\hat H_t$. With $\lambda=0$ the evidence is encoded by the model with an empty patch; we call this \emph{$\Tz$-only encoding} and the writer trained with it the \emph{plain writer}. With $\lambda=1$ the evidence is encoded by the model reading its current memory (\emph{memory-aware encoding}). The \emph{delta writer} is the plain writer forked at 80M training tokens and fine-tuned with $\lambda=1$ for the remaining 20M.

\paragraph{Event features.} For event $a$ at position $p$ and head $h$,
\begin{equation}
  x_a=\big[\hat H_p;\ \mathrm{mean}(\hat H_{p+1:p+4});\ \mathrm{Pool}(\hat H);\ \mathrm{emb}(h)\big],
\end{equation}
the representation at the position, the mean over the valid tokens among the next four positions (a zero vector if there are none, as at the end of a block), an attention pool over the block, and an embedding of the (order, head) pair. The right context lets the writer see what follows the n-gram whose row it is about to write.

\paragraph{Target and gate.}
\begin{equation}
  v_a=s_{h}\,V_\theta(x_a),\qquad
  \beta_a=\sigma\big(B_\theta[\,x_a,\ \Delta_{j(a)},\ v_a-\Delta_{j(a)},\ \log(1+c_{j(a)})\,]\big).
  \label{eq:writer}
\end{equation}
$V_\theta$ and $B_\theta$ are two-layer MLPs with hidden width 256. $V_\theta$ outputs a 160-dimensional target and $B_\theta$ a single gate logit, so $\beta_a$ and $w_j$ are scalars. The fixed per-head scale $s_h$ is calibrated at initialisation so that the initial targets have RMS $\rho\cdot\mathrm{RMS}(\Tz_h)$ with $\rho=0.01$; it is a calibration, not a cap, and the target scale is free to change during training. The target does not read the row state directly. Under memory-aware encoding it still depends on the memory through $\hat H_t$; stopping the gradient removes only the backward path, not this forward dependence. The gate reads the row state directly: it sees the row's current patch, the residual and the write count, so it can learn to write less when the row already agrees with the target and more when the old value should be replaced. The native writer has 4,075,681 parameters.

The writer is tied to the backbone and the table it was trained on. Its input $\hat H$, its target scale and the effect of a patch through the read-out all depend on them; in our 0.6B prototype, a writer trained on a base model kept about one fifth of its gain when moved to the post-trained model.

\subsection{Offline training with the pretraining objective}
\label{sec:method-training}

The writer is trained on raw text with next-token prediction, the same objective as pretraining, and needs no labels or question--answer data.
\begin{enumerate}[leftmargin=*,itemsep=2pt]
  \item A document, or a time-ordered stream of reports on one event, is split into blocks of 512 tokens. Each query carries the last 128 tokens of the previous block as prefix.
  \item Memory starts empty. The first block of an ordinary document is only written; in an event stream the whole first article is only written, and scoring starts with the second article. Every later block is \emph{predicted, then written}: the model reads the memory as it was before the block, the block's cross-entropy is computed, and the block is then written with D1. The last block is not written. Within an event stream, memory persists across articles.
  \item The loss is normalised once by the number of target tokens in the whole training group and backpropagated through the full chain of writes; only $\theta$ receives gradients.
\end{enumerate}
The gradient path is: loss $\to$ frozen layers above the injection $\to$ native read-out \eqref{eq:read1}--\eqref{eq:read3} $\to$ patch $\Delta_t$ $\to$ D1 chain \eqref{eq:d1} $\to$ writer \eqref{eq:writer}. Stopping gradients at $\hat H_t$ removes the branch ``old memory $\to$ encoding $\to$ features'' under memory-aware encoding. This is a truncation of backpropagation through time, chosen for cost: the forward dynamics, and hence the agreement between training and use, are unchanged.

\paragraph{Training and use run the same computation.} The writer and D1 used at training time are exactly those used at use time; training only adds a backward pass around them. At use time evidence is written as it arrives and queries read $\Tz+\Delta$, with no backpropagation, no optimizer and no change to any weight. \Cref{alg:use,alg:train} give both procedures.

\input{sections/alg_writer}

\subsection{Memory state, reversibility and decay}
\label{sec:method-state}

The state stores only rows that have been written: their row ids, patches $\Delta_j$ and write counts $c_j$, plus a block clock with the bookkeeping for lazy decay. It carries no optimizer state. Clearing the state returns the model to its empty-patch read-out. This is exact within one read-out implementation, but the split FP32 read-out of \cref{app:engineering} is not bitwise equal to the original BF16 read-out, so every comparison uses the empty patch in the same implementation as its baseline. A user's state is a self-contained set of sparse rows that can be stored apart from the shared backbone and writer.

D1 already forgets by content: rewriting a row replaces its old value by the factor $1-w_j$. An optional decay adds forgetting by time. Unlike the decay of Gated DeltaNet and KDA~\citep{yang2025gdn,kimi2025linear}, which pulls the state toward zero, ours pulls the patch toward zero and hence the table toward $\Tz$: $\Delta_j\leftarrow\alpha\Delta_j$, i.e.\ $T_j\leftarrow\Tz_j+\alpha(T_j-\Tz_j)$. It is applied lazily. Each block of new evidence is first encoded with the current state; the clock then advances once, and the rows the block addresses enter the gate and D1 at their decayed values. A read uses each row at its decayed value without advancing the clock or storing the result. The main method uses $\alpha=1$ (no decay); \cref{sec:memory} compares $\alpha=0.99$.

%% file: sections/fig_overview.tex
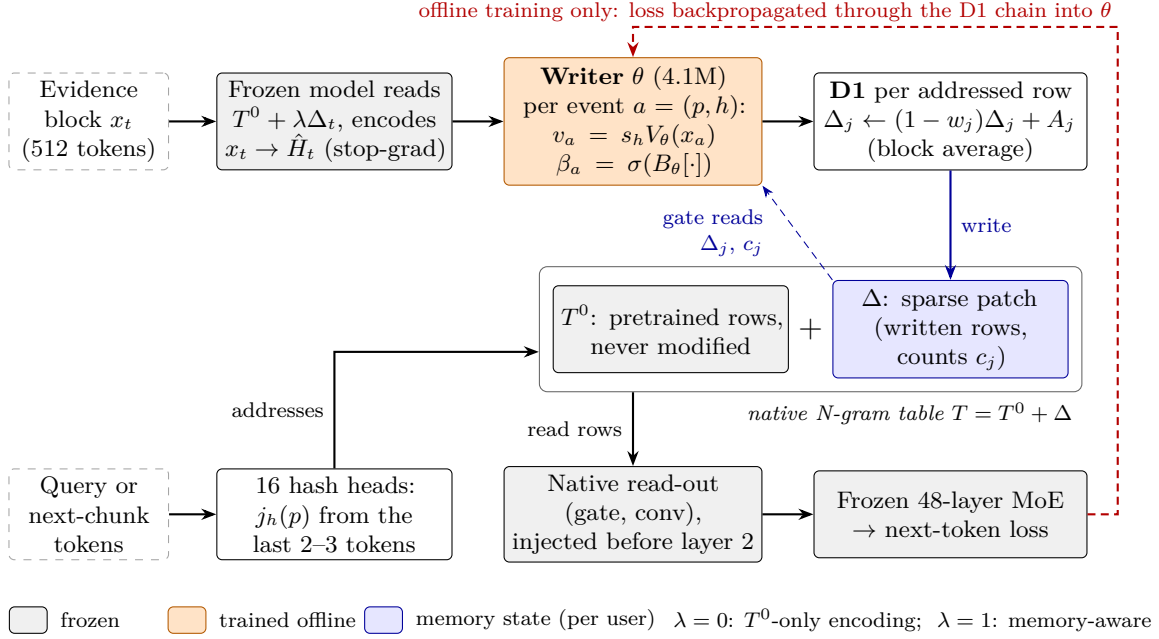
\begin{figure}[t]
  \centering
  \begin{tikzpicture}[
      font=\footnotesize,
      box/.style={draw, rounded corners=2pt, align=center, minimum height=1.15cm, inner sep=3pt},
      frozen/.style={box, fill=gray!12},
      trained/.style={box, fill=orange!22, draw=orange!70!black},
      state/.style={box, fill=blue!10, draw=blue!60!black},
      io/.style={box, draw=gray!70, dashed},
      flow/.style={-{Stealth[length=2mm]}, thick},
      side/.style={-{Stealth[length=1.6mm]}, densely dashed, blue!60!black},
      train/.style={-{Stealth[length=2mm]}, densely dashed, thick, red!70!black},
    ]
    \node[io, text width=1.9cm] (ev) at (0,0) {Evidence\\block $x_t$\\(512 tokens)};
    \node[frozen, text width=2.9cm] (enc) at (3.25,0) {Frozen model reads\\$\Tz+\lambda\Delta_t$, encodes\\$x_t\to\hat H_t$ (stop-grad)};
    \node[trained, text width=3.2cm] (wr) at (7.2,0) {\textbf{Writer} $\theta$ (4.1M)\\per event $a=(p,h)$:\\$v_a=s_hV_\theta(x_a)$\\$\beta_a=\sigma(B_\theta[\cdot])$};
    \node[box, text width=3.4cm] (d1) at (11.4,0) {\textbf{D1} per addressed row\\$\Delta_j\leftarrow(1-w_j)\Delta_j+A_j$\\(block average)};
    \draw[flow] (ev) -- (enc);
    \draw[flow] (enc) -- (wr);
    \draw[flow] (wr) -- (d1);

    \node[frozen, text width=2.9cm] (t0) at (7.7,-2.75) {$\Tz$: pretrained rows,\\never modified};
    \node[state, text width=2.9cm] (delta) at (11.4,-2.75) {$\Delta$: sparse patch\\(written rows, counts $c_j$)};
    \node at ($(t0.east)!0.5!(delta.west)$) {\large $+$};
    \begin{scope}[on background layer]
      \node[draw=black!60, rounded corners=3pt, inner sep=5pt, fit=(t0)(delta), label={[font=\scriptsize\itshape, anchor=north east]south east:native N-gram table $T=\Tz+\Delta$}] (table) {};
    \end{scope}
    \draw[flow, blue!60!black] (d1.south) -- node[right, font=\scriptsize] {write} (delta.north);
    \draw[side] (delta.north west) -- node[pos=0.5, left, font=\scriptsize, align=right, xshift=-9pt] {gate reads\\$\Delta_j$, $c_j$} (wr.south east);

    \node[io, text width=1.9cm] (q) at (0,-5.2) {Query or\\next-chunk\\tokens};
    \node[box, text width=2.9cm] (hash) at (3.25,-5.2) {16 hash heads:\\$j_h(p)$ from the\\last 2--3 tokens};
    \node[frozen, text width=3.2cm] (ro) at (7.2,-5.2) {Native read-out\\(gate, conv),\\injected before layer 2};
    \node[frozen, text width=3.4cm] (bb) at (11.4,-5.2) {Frozen 48-layer MoE\\$\to$ next-token loss};
    \draw[flow] (q) -- (hash);
    \draw[flow] (hash.north) |- node[pos=0.25, left, font=\scriptsize] {addresses} ($(table.west)+(0,-0.3)$);
    \draw[flow] (ro.north |- table.south) -- node[left, font=\scriptsize] {read rows} (ro.north);
    \draw[flow] (ro) -- (bb);

    \draw[train] (bb.east) -- ++(0.4,0) coordinate (rc) |- ($(wr.north)+(0,0.4)$) coordinate (tl) -- (wr.north);
    \node[font=\scriptsize, text=red!70!black, anchor=south east, inner sep=2pt] at (rc |- tl) {offline training only: loss backpropagated through the D1 chain into $\theta$};

    \node[frozen, minimum height=0.35cm, minimum width=0.5cm] (l1) at (-0.8,-6.6) {};
    \node[right=1pt of l1, font=\scriptsize] {frozen};
    \node[trained, minimum height=0.35cm, minimum width=0.5cm] (l2) at (1.3,-6.6) {};
    \node[right=1pt of l2, font=\scriptsize] {trained offline};
    \node[state, minimum height=0.35cm, minimum width=0.5cm] (l3) at (3.9,-6.6) {};
    \node[right=1pt of l3, font=\scriptsize] {memory state (per user)};
    \node[font=\scriptsize, anchor=west] at (7.6,-6.6) {$\lambda=0$: $\Tz$-only encoding; \ $\lambda=1$: memory-aware};
  \end{tikzpicture}
  \caption{Delta N-gram. \emph{Top}: a block of evidence is encoded by the frozen model, the writer emits a target $v_a$ and gate $\beta_a$ for every (position, hash head) event, and D1 updates the addressed rows of the sparse patch $\Delta$; the gate also reads each row's current patch and write count. \emph{Bottom}: every forward pass reads $\Tz+\Delta$ at the hashed addresses of the last two or three tokens through the model's native read-out. At use time only the solid arrows run, forward only; during offline training the next chunk's loss is backpropagated through the whole write chain into the writer (dashed red). Clearing $\Delta$ restores the original model.}
  \label{fig:overview}
\end{figure}

%% file: sections/alg_writer.tex
\begin{algorithm}[t]
\caption{Delta N-gram at use time: forward-only writing and reading.}
\label{alg:use}
\begin{algorithmic}[1]
\Require frozen backbone, read-out and $\Tz$; frozen writer $\theta$; encoding $\lambda\in\{0,1\}$; decay $\alpha$ ($\alpha=1$: none)
\State $\Delta\gets\varnothing$ \Comment{sparse rows $(j,\Delta_j,c_j)$ and a block clock}
\Procedure{Write}{$x_t$} \Comment{one block of new evidence}
  \State $\hat H\gets\mathrm N\big(E(x_t;\,\Tz+\lambda\Delta)\big)$ \Comment{\eqref{eq:encode}; current state, gradient stopped}
  \State advance the clock; take each row the block addresses at its decayed value $\alpha^{\text{elapsed}}\Delta_j$
  \For{each event $a=(p,h)$, $p$ a new token of $x_t$} \Comment{prefix and padding never write}
    \State $j\gets j_h(p)$;\quad build $x_a$ from $\hat H$, $p$ and $h$ \Comment{right mean over valid tokens; zero if none}
    \State $v_a\gets s_hV_\theta(x_a)$;\quad $\beta_a\gets\sigma\big(B_\theta[x_a,\,\Delta_j,\,v_a-\Delta_j,\,\log(1+c_j)]\big)$ \Comment{pre-block state}
  \EndFor
  \For{each row $j$ addressed by the block} \Comment{D1, \eqref{eq:d1}; $n_j$ counts all events on $j$}
    \State $w_j\gets\frac1{n_j}\sum_{a\to j}\beta_a$;\quad $A_j\gets\frac1{n_j}\sum_{a\to j}\beta_av_a$
    \State $\Delta_j\gets(1-w_j)\Delta_j+A_j$;\quad $c_j\gets c_j+1$
  \EndFor
\EndProcedure
\Procedure{Read}{$q$}
  \State run the model on $q$; every position reads $\Tz+\Delta$ through \eqref{eq:read1}--\eqref{eq:read3}, each row at its decayed value
  \State the clock is not advanced and nothing is written back
\EndProcedure
\State To remove a user's memory, clear $\Delta$: the model reads the empty patch again.
\end{algorithmic}
\end{algorithm}

\begin{algorithm}[t]
\caption{Offline training of the writer (native settings). The forward computation is that of \cref{alg:use}; training adds one backward pass.}
\label{alg:train}
\begin{algorithmic}[1]
\For{each update}
  \State $S\gets0$;\quad $N\gets0$ \Comment{summed token loss and target count over the group}
  \For{each of the 12 trajectories in the group} \Comment{a document or an event stream}
    \State split each article separately into 512-token blocks $x_1,\dots,x_L$ \Comment{no block spans two articles}
    \State $\Delta\gets\varnothing$ \Comment{memory persists across articles of one stream}
    \For{$t=1,\dots,L$}
      \If{$x_t$ is scored} \Comment{documents: $t\ge2$; event streams: from the second article}
        \State query prefix: the previous 128 tokens of the same article, or BOS if $x_t$ starts an article
        \State $S\gets S+{}$summed cross-entropy of $x_t$ reading $\Tz+\Delta$;\quad $N\gets N+|x_t|$ \Comment{memory before $x_t$}
      \EndIf
      \If{$t<L$} \textsc{Write}$(x_t)$ \Comment{the last block is not written}
      \EndIf
    \EndFor
  \EndFor
  \State $\mathcal L\gets S/N$ \Comment{normalised once, across all GPUs}
  \State backpropagate $\mathcal L$: frozen layers $\to$ read-out $\to\Delta\to$ D1 chain $\to\theta$
  \State update $\theta$ with AdamW (learning rate $3\times10^{-4}$, no weight decay, clipping at 1)
\EndFor
\end{algorithmic}
\end{algorithm}

%% file: sections/04_analysis.tex
\section{Analysis}
\label{sec:analysis}

This section states what the construction guarantees and what it does not. \Cref{sec:analysis-d1} bounds the state. \Cref{sec:analysis-drift} carries the drift bound of~\citep{wu2026drift} over to the patch. \Cref{sec:analysis-sep} shows that, for the plain writer, the effect of removing any set of rows can be read off the final state. \Cref{sec:analysis-mem} analyses the dynamics that memory-aware encoding introduces.

\subsection{Boundedness of D1}
\label{sec:analysis-d1}

\begin{proposition}[Boundedness]
\label{prop:bounded}
If every target satisfies $\norm{v_a}\le V$, then D1 gives $\norm{\Delta_j^{+}}\le\max(\norm{\Delta_j},V)$ for every row. Starting from $\Delta=0$, $\norm{\Delta_j}\le V$ at all times.
\end{proposition}
\begin{proof}
If $w_j=0$ then $A_j=0$ and the row is unchanged. Otherwise $\Delta_j^{+}=(1-w_j)\Delta_j+w_j\bar v_j$ with $\bar v_j=\sum_a\beta_av_a/\sum_a\beta_a$, a convex combination of $\Delta_j$ and a convex combination of targets, since $w_j\in(0,1]$.
\end{proof}

For a fixed writer the bound $V$ exists under both encodings. $\hat H$ is RMS-normalised, the right-context mean and the attention pool are convex combinations of normalised vectors, and there are 16 head embeddings, so the features lie in a compact set and $V_\theta$ is bounded on it. The bound holds for each fixed $\theta$, not uniformly over training.

Averaging over the $n_j$ events, rather than summing, is what makes this hold. With a sum, $w_j=\sum_a\beta_a$ can far exceed 1: twenty events with $\beta=0.25$ give $1-w_j=-4$, and repeated application diverges. Averaging bounds the aggregate strength of repeated hits on a row. It is not a key normalisation, since one-hot keys already have unit norm.

Two things the proposition does not give. First, the backward pass is not a contraction: because the gate reads the old patch, the state Jacobian gains a term $\frac1{n_j}\sum_a(v_a-\Delta_j)(\nabla_{\Delta}\beta_a)^\top$ that can exceed one, so training clips gradients. Second, $V$ is set by the writer's targets, not by $\norm{\Tz_j}$. The proposition bounds the state; it does not bound drift relative to the pretrained values. A row written by every example is pulled toward the writer's targets whatever its pretrained norm, which is how the hot rows of \cref{sec:rows} reach 4--12 times $\norm{\Tz_j}$.

\subsection{From row displacement to representation drift}
\label{sec:analysis-drift}

We follow the notation of~\citep{wu2026drift}. $F^t(q)$ is the representation of a fixed context $q$ at a measured site (the LM-head input unless stated otherwise) after the state has absorbed $t$ tasks. Let $\Psi$ be the frozen map from the injected input to that site, and $\Phi$ the read-out \eqref{eq:read1}--\eqref{eq:read3} from table rows to $u$.

The bound below is conditional, with one condition for each map.
\begin{itemize}[leftmargin=*,itemsep=2pt]
  \item \textbf{Read-out.} The stacked rows $e$ are linear in the table, so the straight segment between the two tables is also a straight segment in $e$. Along it, at every point and not only at its ends, every gate inner product in \eqref{eq:read2} keeps its sign; there $\Phi$ is differentiable.
  \item \textbf{Backbone.} $\Phi$ is nonlinear, so the image of that segment is a curve in $u$, and controlling $\Psi$ along the curve does not give a bound between its endpoints. We therefore assume that $\Psi$ is Lipschitz with constant $L_\Psi$ on a convex set containing the straight segment between the two injected inputs $u^i_{\le q}$ and $u^t_{\le q}$. This holds, for instance, when every token is routed to the same experts at every point of that segment.
\end{itemize}
Where these conditions fail, the bound needs additional jump terms. Because of its $10^{-6}$ floor, $\phi$ jumps slightly at zero, and a routing switch moves the output by a finite amount, so a uniform $L_\Psi$ is not a property of the architecture. Under the conditions, with a single injection site,
\begin{equation}
  \norm{F^t(q)-F^i(q)}\le L_\Psi\,\norm{u^t_{\le q}-u^i_{\le q}},
  \label{eq:lip}
\end{equation}
where the right-hand side covers the whole prefix because attention reads all of it. Let $r$ be the number of convolution taps and $\delta$ their dilation ($r=4$, $\delta=3$ for Qwen3.8), so that $u_p$ depends on $e_p,e_{p-\delta},\dots,e_{p-(r-1)\delta}$. Let $J_{p,\tau}$ ($\tau=0,\dots,r-1$) be the Jacobian of $u_p$ with respect to $e_{p-\delta\tau}$ along the same segment, and let $C_J^2=\max_p\sum_\tau\sup\norm{J_{p,\tau}}^2$, with the supremum taken over the segment. By the mean-value inequality and Cauchy--Schwarz, $\norm{\Delta u_p}^2\le C_J^2\sum_\tau\norm{\Delta e_{p-\delta\tau}}^2$, and each $\Delta e_{p'}$ enters at most $r$ of the $\Delta u_p$, namely those with $p=p'+\delta\tau$. Averaging \eqref{eq:lip} over the probe set $Q_i$ of task $i$ gives
\begin{equation}
  \big(D_i^t\big)^2:=\frac1{|Q_i|}\sum_{q\in Q_i}\norm{F^t(q)-F^i(q)}^2
  \;\le\;r L_\Psi^2C_J^2\sum_j m_{ij}\,\norm{\Delta T_j}^2,
  \label{eq:drift-bound}
\end{equation}
\begin{equation*}
  m_{ij}=\frac{\#\{(q,p,h):q\in Q_i,\ p\le q,\ j_h(p)=j\}}{|Q_i|},
\end{equation*}
with $\Delta T_j=T^t_j-T^i_j$. The factor $m_{ij}$ is the average number of times a probe's whole causal prefix reads row $j$. The bound is \emph{prefix overlap times row displacement}: a write rule cannot change the overlap, only the displacement.

Under \eqref{eq:patch}, both references are explicit:
\begin{itemize}[leftmargin=*,itemsep=2pt]
  \item \textbf{Big-$K$ drift} (general probes $Q_0$, reference $\Tz$): $\Delta T_j=\Delta^t_j$, the patch itself. A general input whose prefix reads no written row is exactly unchanged. Drift on general inputs therefore comes from rows that are both read by them ($m_{0j}>0$) and written; template and format n-grams are the candidates.
  \item \textbf{Small-$k_i$ drift} (task-$i$ probes, reference: the state just after task $i$): $\Delta T_j=\Delta^t_j-\Delta^i_j$. A row that has already converged to the writer's target when task $i$ finishes, and receives similar targets afterwards, contributes little, however large its patch is.
\end{itemize}
The two can therefore move in opposite directions. A format row that every example writes reaches its displacement early and then changes little: this inflates big-$K$ drift on every input that reads it but adds little small-$k_i$ drift afterwards. \Cref{sec:longchain,sec:rows} observe exactly this pattern.

\paragraph{Comparison with gradient table updates.} Under a per-row Adam step, as in~\citep{wu2026drift}, the coordinate step $\alpha\hat m/(\sqrt{\hat v}+\epsilon)$ is independent of the gradient's scale. Noise therefore still moves a row whose mean gradient is zero. As a random-walk approximation, valid when the normalised updates are nearly unbiased, of stable variance and only short-range correlated, the displacement after $S$ steps is about $c\,\alpha\sqrt{dS}$, or $c\sqrt{d\sum_s\alpha_s^2}$ with a varying learning rate. With the dense Adam of~\citep{wu2026drift}, momentum also keeps moving a row for about $1/(1-\beta_1)$ steps after it leaves the batch. This tail does not carry over to the sparse Adam of our prototype baseline, which updates only rows present in the batch. Under D1 the displacement is set by the residual $v_a-\Delta_j$ and the learned gate, with no optimizer state or momentum. Whether the writer actually gives $v_a\approx\Delta_j$ or a small $\beta_a$ on rows that should be kept depends on what it learned; the formula does not guarantee it.

\subsection{Row separability}
\label{sec:analysis-sep}

The plain writer has a property that makes row-level interventions cheap.

\begin{proposition}[Row separability]
\label{prop:separable}
Consider the plain writer ($\lambda=0$), a fixed evidence sequence, the same clock and decay, and an initial state with $\Delta=0$. Let $S$ be any set of rows. Compare (A) running the chain and then setting $\Delta_j=0$ for $j\in S$, with (B) running the chain without committing any write to rows in $S$. The final patches of (A) and (B) are equal, so every query reads the same table.
\end{proposition}
\begin{proof}
Under $\lambda=0$ the encoding $\hat H_t$ does not depend on $\Delta$, so the features $x_a$ and targets $v_a$ are identical in (A) and (B). The gate $\beta_a$ reads only the patch and count of its own row $j(a)$, and D1 updates row $j$ only from events addressed to $j$. By induction over blocks, every row outside $S$ has the same patch and count in (A) and (B). Rows in $S$ end at zero in both.
\end{proof}

The effect of removing any set of rows on drift and on gain can therefore be computed by zeroing them in the final state, without rerunning the write chain. The equivalence covers removal only. A cap applied at every write is a different intervention, because a row's later updates depend on its own history. Take $w=\tfrac12$, targets $4$ and then $-2$, and a cap $|\Delta|\le1$: the uncapped row goes $0\to2\to0$ and stays $0$ when its final value is capped, while a row capped at every write goes $0\to1\to-\tfrac12$. Evaluating a cap enforced during writing therefore requires replaying the write chains of the capped rows; for the plain writer, only those rows change. The equivalence is also limited in three further ways. It concerns the patch, not the full state: write counts on $S$ differ. It does not extend to trajectories that keep writing after the intervention, because the gate sees those counts. And it fails for the delta writer, whose encoding reads every row in the block and therefore couples rows.

\subsection{Memory-aware encoding: anchoring, feedback and reproducibility}
\label{sec:analysis-mem}

With $\Tz$-only encoding a target depends only on the evidence, so the patch always lies in the convex hull of $\{0\}$ and the evidence's targets: the writer can reweight what the evidence says but cannot invent new targets. This anchoring does not imply a unique fixed point. The gate reads the old patch, so the effective target $\bar v(\Delta)=\sum_a\beta_a(\Delta)v_a/\sum_a\beta_a(\Delta)$ depends on the memory even when the $v_a$ do not. Take one coordinate, two events with $v_\pm=\pm1$ and gates $\beta_\pm=\sigma(\pm4d)$. Then $w=\tfrac12$ and
\begin{equation}
  F(d)=\tfrac12d+\tfrac12\tanh(2d),
\end{equation}
which has three fixed points, $0$ and $\pm0.957$; the one at zero is unstable ($F'(0)=1.5$). The gate favours evidence that agrees with the memory, and the memory then locks in. This is a constructed example, not a measurement on the native model, but it shows that fixed targets plus a convex update do not rule out self-reinforcement through the gate. The case most exposed to it is revision, where old and new values conflict on the same row.

Memory-aware encoding adds a second route: the target itself depends on the state. We analyse the local map with the evidence, the write counts, the clock and all other external inputs held fixed. In the actual chain a written row's count grows by one at every write and the gate reads $\log(1+c_j)$, so each step applies a different map. Stability over a long chain depends on products of successive Jacobians, which the eigenvalues of a single step do not determine. For a single event with fixed evidence and count, $T(s)=s+\beta(s)\big(v(s)-s\big)$, and at a differentiable fixed point $v(s^\star)=s^\star$,
\begin{equation}
  DT(s^\star)=(1-\beta)I+\beta\,Dv(s^\star).
  \label{eq:fixed-point}
\end{equation}
With exogenous targets ($Dv=0$, as in a DeltaNet layer or our $\Tz$-only writer) this is a pure contraction by $1-\beta$. In the scalar case, $|1-\beta+\beta\mu|<1$, i.e.\ $1-2/\beta<\mu<1$ for an eigenvalue $\mu$ of $Dv$, is sufficient for linearised stability. At equality the linearisation is inconclusive: $T(s)=\tfrac12s+\tfrac12\tanh s$ has $T'(0)=1$, yet $0$ is asymptotically stable. For a row written by several events,
\begin{equation}
  DF_j=(1-w_j)P_j+\frac1{n_j}\sum_{a\to j}\Big[\beta_a\,Dv_a+(v_a-\Delta_j)(\nabla\beta_a)^\top\Big],
  \label{eq:row-jacobian}
\end{equation}
where $P_j$ selects row $j$ from the full state. The first bracketed term is the \emph{feedback} that memory-aware encoding adds; it couples rows, because the encoding reads every row in the block. The second is the \emph{gate reweighting} present under both encodings; since the targets of different events differ, it does not vanish at a fixed point.

The gain along one direction decides nothing. $J=\mathrm{diag}(0.5,2)$ contracts one direction and expands another, while negative feedback can contract despite a large gain: with $w=0.25$ and $\bar v(d)=-\tanh(2d)$, the gain at the origin is 2, yet $F'(d)=0.75-0.5\,\mathrm{sech}^2(2d)\in[0.25,0.75]$ and the map contracts globally. Rows a block does not touch contribute an identity block. Restricted to the rows it updates, a spectral radius of $DF$ below one at a fixed point is sufficient for local asymptotic stability of the map with frozen inputs. On a closed, invariant, complete set, a uniform contraction bound $L<1$ is sufficient for a unique, globally attracting fixed point. A Lipschitz bound alone is not, and strict contraction is not necessary for uniqueness. Neither can be checked analytically for a 48-layer MoE encoder, so we measure.

\paragraph{Encoding-substitution sensitivity.} To screen the feedback term we hold the direct state fixed and replace only the encoding:
\begin{equation}
  G=\frac{\norm{\bar v(\hat H_{\mathrm{mem}},\Delta)-\bar v(\hat H_{\mathrm{base}},\Delta)}_F}{\norm{\Delta\ \text{on the rows read by the block}}_F}.
  \label{eq:G}
\end{equation}
Replacing the encoding changes both the targets and the gates, so $G$ mixes a change of targets with gate reweighting and does not isolate the term $\beta\,Dv$. Blocks whose rows all have a zero patch, where the denominator vanishes, are excluded. $G$ is a secant, used as a screen rather than as a proof of contraction. \Cref{sec:memory} reports it together with twin-trajectory measurements of whether small state differences grow.

\paragraph{Reproducibility.} The analysis above is differentiable. A sparse MoE encoder is not: an arbitrarily small change of the memory can switch the experts of some tokens and move $v$ by a finite amount. This makes the memory-aware state sensitive to perturbations. On its own it does not prevent identical runs from agreeing bit for bit, but our GPU kernels are not bitwise deterministic: rebuilding $\Delta$ from the same $\hat H$ differs by about $10^{-9}$. Under memory-aware encoding such differences grow to a plateau, and the useful question is whether they stay bounded and below the read-side noise. $\Tz$-only encoding is reproducible only for a fixed batch layout, since batching changes the encoder's numerics as well.

%% file: sections/05_setup.tex
\section{Experimental Setup}
\label{sec:setup}

\subsection{Models and writer training}
\label{sec:setup-models}

\paragraph{Prototype in two stages.} Before the native model we develop the method on a frozen 0.6B model with an external, zero-initialised hash table ($\Tz=0$): 16,384 rows each for bigrams and trigrams, row width 1024. The two orders are summed, divided by $\sqrt2$ and injected before decoder layer~1. The prototype went through two stages that differ in the training objective.
\begin{itemize}[leftmargin=*,itemsep=2pt]
  \item \textbf{Question-answering objective.} On Qwen3-0.6B-Base, the writer is trained on synthetic episodes. Short evidence blocks stating facts about named entities are written one at a time, and the loss is placed on the answers to template questions about them. Training uses 2,000 episodes (500 outer steps, seed 17), with disjoint entity pools for training and evaluation. Evaluation generates answers greedily on 16 held-out worlds with two writes each. It compares the writer with online gradient writes into the same table (ten SparseAdam steps per evidence block, learning rate 0.01 or 0.05), the full evidence in context, and BM25 retrieval of the top two blocks. Transfer probes, without retraining, test a paraphrased question that keeps the entity name, chains of 4 and 16 writes, and revision of an earlier value.
  \item \textbf{Raw-text objective.} The writer is then trained as in \cref{sec:method-training}, on raw text with no question--answer data, on the post-trained Qwen3-0.6B (1,190,465 writer parameters, hidden width 128). A 5M-token code-only pilot with two seeds comes first, followed by a 100M-token run on the five-source mixture described below. A matched 5M-token run on the base model tests whether a writer transfers between the two readers.
\end{itemize}
Prototype and native results are reported separately throughout.

\paragraph{Native model.} The main model is Qwen3.8-Flash-Next~\citep{qwen38flashnext}: 48 layers, 512 routed experts with ten active per token, and the native table and read-out of \cref{sec:method-native}. Backbone, read-out and $\Tz$ are frozen; the writer (4,075,681 parameters) is the only trained component. We train on four H200 GPUs with expert parallelism 4 and data parallelism 4. Each update uses 12 complete trajectories and backpropagates through the whole write chain of each. The optimizer is AdamW with learning rate $3\times10^{-4}$, no weight decay and gradient clipping at 1. Two arms start from the same initial writer: $\alpha=1$ is trained to 99,997,072 target tokens (4,172 updates) and $\alpha=0.99$ to 39,998,107 (1,669 updates). The $\alpha=1$ budget was extended from a shared to a per-arm 100M after the 20M and 40M development evaluations had been seen, so the 60--100M points are exploratory. Implementation details are in \cref{app:engineering}.

\paragraph{Training data.} Raw English text and source code from six streams, mixed by target-token share: code from Stack v2 Edu (20\%), arXiv papers (25\%) and Project Gutenberg books (20\%) from the Common Pile~\citep{kandpal2025commonpile}, FineWeb-Edu web pages (15\%)~\citep{penedo2024fineweb}, CC-News articles (17.5\%)~\citep{ccnews}, and time-ordered streams of several CC-News reports on the same event (2.5\%). The corpus has 38,566 training documents and 713 event streams backed by 2,002 articles; 3,069 documents and 224 event streams are held out for validation. Trajectories are 2 to 32 blocks long depending on the source. Memory starts empty for every trajectory and persists across articles within an event stream. No question--answer data or labels are used.

\paragraph{Delta writer.} The memory-aware writer is forked from the $\alpha=1$ checkpoint at 80M and trained to 100M on exactly the 835 update groups that the main line used between 80M and 100M. Only the encoding changes ($\lambda=1$); all writer parameters are trained and the AdamW state is inherited. The main line over the same groups is the matched control.

\subsection{DEV64: memory of new events}
\label{sec:setup-dev64}

DEV64 consists of 64 event families from the second half of 2025, drawn from 146 source articles. Each family has an earlier and a later real news report on the same event. Memory is built from the earlier report, then $k\in\{0,1,4,8,16\}$ unrelated articles are written, and finally the later report is scored by teacher-forced continuation. The scored report receives the true first 128 tokens as context and is never written into memory. Three conditions differ only in what was written:
\begin{itemize}[leftmargin=*,itemsep=2pt]
  \item \textbf{self}: the family's own earlier report, followed by the $k$ distractors;
  \item \textbf{donor}: an earlier report from another family, paired by a hash fixed before training, followed by the same $k$ distractors;
  \item \textbf{empty}: nothing.
\end{itemize}
We report two differences in nats per target token:
\begin{equation}
  \Dempty=L_{\mathrm{empty}}-L_{\mathrm{self}},\qquad
  \Ddonor=L_{\mathrm{donor}}-L_{\mathrm{self}}.
\end{equation}
$\Dempty$ is the total benefit of the document's memory, including generic help; $\Ddonor$ is the part specific to the document. Forgetting can raise $\Ddonor$ by making the donor worse, so the two are always reported together. Losses are averaged within each family and then equally over families.

Three controls accompany the main numbers. The \emph{first scoring window}, the up to 384 targets right after the 128-token context, tests whether the gain is present before the later report has unfolded. A \emph{prefix split} separates families whose 128-token context restates the early facts (34), does not (23) or is unclear (7). A \emph{direction-permutation control} permutes the channels of the read-out vector at scoring time, preserving addresses, counts and row norms, to test whether the gain comes from what was written.

Intervals are percentile intervals from 10,000 family-paired bootstrap resamples; comparisons across checkpoints resample the same families. Unless stated otherwise they are pointwise and uncorrected. The weight $\times$ encoding factorial of \cref{sec:memory-encoding} uses a pre-registered analysis with 100,000 resamples and Bonferroni correction over its 60 comparisons. The pretraining cutoff of Qwen3.8 is not documented, so the model may have seen these events; the empty condition measures what it already knows.
For the prototype we also report a historical held-out set of 640 families from the five sources, scored the same way and macro-averaged over sources.

\subsection{Seven-task chain}
\label{sec:setup-seven}

We reuse the seven morphological inflection tasks of~\citep{wu2026drift}, drawn from SIGMORPHON~\citep{goldman2023sigmorphon}, in the order SME $\to$ NAV $\to$ HYE $\to$ AFB $\to$ SQI $\to$ MKD $\to$ AMH. Each task contributes 5,000 support examples (prompt and reference answer, in the chat template with thinking disabled), the same examples and exposures as the 125 Adam steps $\times$ batch 40 of~\citep{wu2026drift}. Every example is one evidence document and one D1 update, so a chain makes 35,000 writes. One state is kept for the whole chain: it is never cleared or partitioned at task boundaries, and the writer is frozen and was never tuned on these tasks. The tasks' earlier results motivated this work, so this is a re-test on historical tasks, not a fresh benchmark.

Evaluation follows~\citep{wu2026drift}:
\begin{itemize}[leftmargin=*,itemsep=2pt]
  \item \textbf{Accuracy}: 64 fixed test items per task, greedy decoding of at most 128 new tokens, exact match after stripping and Unicode NFC normalisation. Each task is also tested \emph{before} its support examples are written, so that acquisition is measured as the change from writing.
  \item \textbf{Answer NLL}: teacher-forced NLL of the reference answers.
  \item \textbf{Drift}: representations at the last prompt token at all 48 layers (four residual streams concatenated, 10,240 dimensions) and at the LM-head input (2,560 dimensions), as relative RMS change. Big-$K$ drift uses 128 math and 64 code probes against the empty state; small-$k_i$ drift compares each task's probes with the state right after that task was written.
  \item \textbf{Seen-example recall}: 448 support examples already written (64 per task) are answered from the final state.
\end{itemize}
The three arms of~\citep{wu2026drift} (N-gram table trained with Adam, shared LoRA, merged LoRA) are reused, not rerun.

Accuracy and the other metrics come from two independent runs. Accuracy is measured on A800 GPUs with vLLM, whose Qwen3.8 read-out we extend with a separate FP32 patch branch. NLL and representations are measured on H200 GPUs with the native training stack. Each machine writes its own chain from an empty state, and the two states differ by 4--6\% because of cross-stack numerics. Each writer (plain and delta) is run once per machine.

\subsection{LoCoMo}
\label{sec:setup-locomo}

LoCoMo~\citep{maharana2024locomo} has 10 long dialogues and 1,986 questions; we use the delta writer. Each dialogue is written once into an empty memory, in session order, rendered as text with a dated header per session and one ``Speaker: utterance'' line per turn, in 512-token blocks. All questions of a dialogue are answered from the same written state, which questions do not modify. Four conditions use the same question prompt: \textbf{self} (own dialogue in memory, only the question in context), \textbf{empty}, \textbf{donor} (another dialogue of similar length in memory) and \textbf{full context} (empty memory, the whole dialogue in context). The prompt instructs the model to answer ``No information available'' when it lacks the information. Decoding is greedy with at most 32 new tokens. The three memory conditions read through the FP32 split read-out of \cref{app:engineering}. The full-context condition needs no patch and runs the native BF16 read-out with expert offloading, so the two kinds of condition use numerically different read paths. The main score is LoCoMo's category-aware token F1 on the 1,540 questions of categories 1--4; the 446 adversarial questions are excluded. We also score the teacher-forced NLL of the reference answers. We do not compare these numbers directly with~\citep{memt2026}, which uses a different backbone, prompts and judging.

%% file: sections/06_memory.tex
\section{Document-Specific Memory from Raw Text}
\label{sec:memory}

On the native model the writer only ever sees raw-text continuation. This section first traces how the prototype arrived at that objective. It then asks whether raw text is enough for the writer to learn memory that is specific to the document written, robust to later writes, and carried by the content of the patch.

\subsection{Prototype I: a question-answering objective}
\label{sec:memory-proto}

\begin{table}[t]
  \centering
  \small
  \caption{Prototype I, 16 held-out two-write worlds (Qwen3-0.6B-Base). \emph{First}: the first fact asked after the first write; \emph{Retained}: the same fact asked after the second write; \emph{New}: the second fact. Greedy generation, first-word score. Empty memory answers 2/16.}
  \label{tab:proto-qa}
  \begin{tabular}{lccc}
    \toprule
    Method & First & Retained & New \\
    \midrule
    Forward writer + D1 & 16 & 16 & 15 \\
    Online gradient writes (SparseAdam, lr 0.01) & 16 & 2 & 16 \\
    Online gradient writes (SparseAdam, lr 0.05) & 16 & 4 & 16 \\
    Full evidence in context & 16 & 16 & 14 \\
    BM25 retrieval, top 2 & 16 & 16 & 14 \\
    \bottomrule
  \end{tabular}
\end{table}

Trained to answer questions, the forward writer binds facts and keeps them (\cref{tab:proto-qa}). One write raises answer accuracy from the empty-memory 2/16 to 16/16. After a second write all 16 first facts are still answered, and 15 of the 16 new ones. Online gradient writes into the same table learn each new fact but overwrite the old one, keeping only 2--4 of 16. The forward writer matches putting the evidence in context and retrieving it with BM25. This is the contrast that motivates the method: writing by gradient into shared rows forgets, while a learned forward write need not.

The same experiments expose the limits that reappear on the native model.
\begin{itemize}[leftmargin=*,itemsep=2pt]
  \item \textbf{The binding is tied to the trained question form.} Asked with a paraphrased question that still names the entity, the writer stays at the empty-memory 2/16 before and after writing. The rows read by such queries overlap the written rows in only about 6\% of positions.
  \item \textbf{Longer chains erode it.} Without retraining, four writes leave 60 of 64 facts answerable. After sixteen writes, 46 of 64 are answered right after their own write but only 19 at the end. Revised values are followed immediately in 5 of 16 worlds.
  \item \textbf{Template rows carry a trade-off.} At an earlier checkpoint, two old answers were lost after the second write. Restoring the ten rows of the literal question template to their values before that write rescues both old answers (14 to 16), but new answers fall from 12 to 2. These are the same shared rows that, on the native model, carry both gain and drift (\cref{sec:rows}).
\end{itemize}
A question-answering objective also needs labelled episodes in a fixed format. We therefore moved to the pretraining objective, which needs no labels and applies to any text.

\subsection{Prototype II: raw text}
\label{sec:memory-proto-raw}

On a 5M-token code-only pilot, two seeds agree: $\Dempty=0.036$ and $0.037$, $\Ddonor=0.024$ and $0.026$, all with intervals above zero. The writer is tied to its reader. A writer trained against the base model and used with the post-trained model keeps a gain of 0.0066 over empty memory, against 0.0363 for the writer trained on the post-trained model itself.

Trained on the five-source mixture for 100M target tokens, the prototype writer yields memory with a clear document-specific component. On the historical held-out set, the five-source macro averages are $\Dempty=0.144$ [0.141, 0.147] and $\Ddonor=0.097$ [0.095, 0.100], positive for every source. On the first block of a later article in an event stream, scored before that article is written, $\Ddonor=0.110$: memory carries across articles. On DEV64 the 100M prototype gives $\Dempty=0.060$ and $\Ddonor=0.038$ without distractors. After 16 distractors $\Ddonor$ falls to 0.012, 31\% of its undistracted value. The small table is the likely bottleneck. At $k=16$ the number of distinct addresses is already 26--31\% below the number of distinct n-gram keys, although no table is more than 62\% occupied. Most of the prototype's DEV64 gain is already present after 5M tokens.

\subsection{Native model}
\label{sec:memory-native}

\begin{figure}[t]
  \centering
  \includegraphics[width=\linewidth]{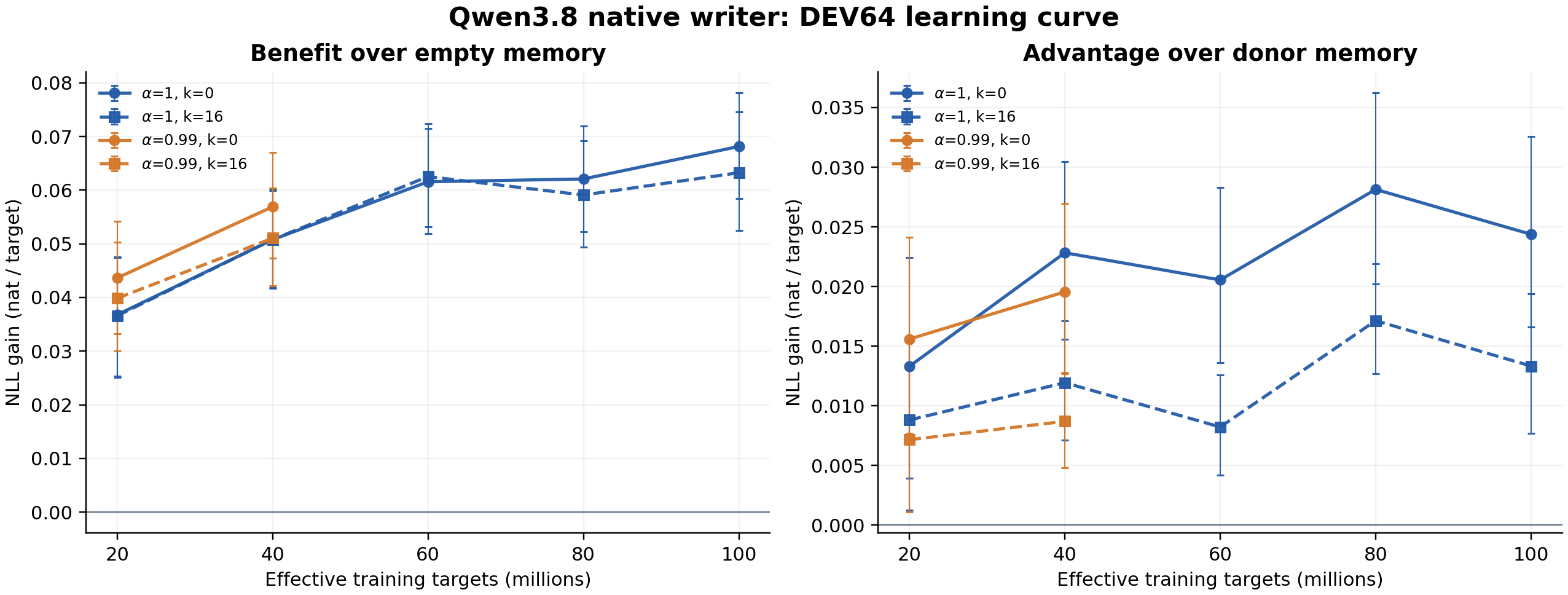}
  \caption{DEV64 during native training. Left: benefit over empty memory ($\Dempty$); right: advantage over a donor document's memory ($\Ddonor$), without distractors ($k=0$) and after 16 distractor articles ($k=16$). Bars are 95\% family-bootstrap intervals.}
  \label{fig:dev64-curve}
\end{figure}

\begin{table}[t]
  \centering
  \small
  \caption{DEV64 at the final native checkpoint ($\alpha=1$, 100M target tokens). Nats per target token, 95\% family-paired bootstrap intervals. Empty-memory NLL is 1.8409 at every $k$.}
  \label{tab:dev64-100m}
  \begin{tabular}{rccc}
    \toprule
    Distractors $k$ & $\Dempty$ & $\Ddonor$ & self NLL \\
    \midrule
    0  & 0.0682 [0.0584, 0.0781] & 0.0244 [0.0166, 0.0326] & 1.7728 \\
    1  & 0.0658 [0.0562, 0.0754] & 0.0187 [0.0124, 0.0256] & 1.7751 \\
    4  & 0.0648 [0.0545, 0.0754] & 0.0167 [0.0113, 0.0226] & 1.7761 \\
    8  & 0.0659 [0.0561, 0.0761] & 0.0170 [0.0119, 0.0225] & 1.7750 \\
    16 & 0.0633 [0.0524, 0.0746] & 0.0133 [0.0077, 0.0194] & 1.7777 \\
    \bottomrule
  \end{tabular}
\end{table}

On the native model the ordering self $<$ donor $<$ empty in NLL holds at every distractor distance and every training stage (\cref{fig:dev64-curve}). \Cref{tab:dev64-100m} gives the final checkpoint. Reading the document's own memory lowers continuation NLL by 0.068 nats/token relative to empty memory and by 0.024 relative to another document's memory, with all intervals above zero.

\paragraph{Robustness to later writes.} After 16 unrelated articles, 93\% of $\Dempty$ remains (0.063) and 55\% of $\Ddonor$ (0.013). The decline is significant for both ($-0.0049$ [$-0.0101$, $-0.0001$] and $-0.0110$ [$-0.0194$, $-0.0031$], paired within family). It is much smaller than on the prototype, where $\Ddonor$ kept 31\%.

\paragraph{The gain is present early and comes from the written content.} In the first scoring window, before the later report has unfolded, $\Dempty=0.052$ and $\Ddonor=0.018$ at $k=0$, and 0.046 and 0.010 at $k=16$. Families whose 128-token context does not restate the early facts still gain ($\Ddonor=0.020$ at $k=0$, 0.012 at $k=16$, both intervals above zero). Permuting the channels of the read-out vector costs the self condition more than the donor condition: the difference of the two permutation costs is 0.028 [0.020, 0.037] at $k=0$ and 0.017 [0.012, 0.023] at $k=16$. The direction of what was written carries information specific to the document.

\paragraph{Training dynamics.} The two quantities develop differently with training. $\Dempty$ keeps rising: 0.037, 0.051, 0.062, 0.062 and 0.068 at 20, 40, 60, 80 and 100M, and from 40M to 100M self NLL improves by 0.017 [0.013, 0.023]. $\Ddonor$ rises from 20M to 40M (0.013 to 0.023) and then fluctuates between 0.021 and 0.028. Its change from 40M to 100M is not significant at any distance, because donor NLL improves almost as much as self NLL. Continued raw-text training thus mainly strengthens help that transfers across documents, while the document-specific part plateaus after about 40M tokens. The checkpoint with the highest $\Ddonor$ is 80M (0.028 at $k=0$, 0.017 at $k=16$). We report the 100M endpoint, fixed by the budget, rather than select on DEV64. The absolute gains are smaller than on the prototype, as expected for a much stronger backbone whose empty-memory NLL is already 1.84 and which may have seen these events.

\subsection{Decay toward the pretrained table}
\label{sec:memory-alpha}

On the native model, $\alpha=0.99$ gives more total benefit at short distances. At equal budgets of 20M and 40M its $\Dempty$ is higher than that of $\alpha=1$ for $k\le8$ (at $k=0$: by 0.0069 [0.0030, 0.0111] and 0.0061 [0.0024, 0.0101]); at $k=16$ the difference is not resolved. Differences in $\Ddonor$ are not significant at either budget. Decay thus trades persistence for help on recent material. Since persistence is the point of the method, the main configuration is $\alpha=1$.

The prototype shows why such comparisons need both metrics. A decay rate predicted per evidence block collapses below 0.4, so memory becomes roughly a cache of the last block. $\Ddonor$ then rises by 20\% while self NLL gets worse and the cross-article gain halves (0.091 to 0.049). Forgetting inflates $\Ddonor$ by degrading the donor.

\subsection{Memory-aware encoding}
\label{sec:memory-encoding}

\paragraph{Stability.} We measure the quantities of \cref{sec:analysis-mem} on the 80M writer, using the next 12 training trajectories, six of which write 4--15 times. The encoding-substitution sensitivity $G$ (\cref{eq:G}) has median 0.30, 99th percentile 0.57 and maximum 0.59 over 49 blocks, and stays in the same range during fine-tuning (median 0.27--0.42). Twin trajectories with memory-aware encoding, started from states that differ by $10^{-8}$ to $10^{-2}$ of their norm, reach a plateau of 1.7--5.3\% relative difference within 2--4 writes and stay there for up to 15 writes. The plateau does not depend on the initial difference, which is what the routing discontinuity of \cref{sec:analysis-mem} predicts; we have not confirmed that mechanism with fixed routing. With $\Tz$-only encoding a 0.1\% perturbation shrinks to about 0.89 of its size, and independent repeats agree to within $10^{-8}$. The effect of either kind of difference on query NLL stays within the read-side numerical noise, which is present under both encodings (after a 1\% perturbation, mean change $+0.0005$ and $-0.0041$ nats/token on 512-token queries). Batch layout matters as well: encoding the same evidence in a different batch layout changes the writer's features by a median 7.6\%.

\paragraph{Training.} Fine-tuned with memory-aware encoding on the same 835 update groups as the main line, the delta writer's training loss is slightly worse throughout (paired $+0.00059$ [0.00041, 0.00070] nats/token). On DEV64, $\Ddonor$ rises by 0.006--0.010 and $\Dempty$ falls by 0.009--0.017 relative to the main line. Weights, encoding at evaluation and batch layout all changed at once, so we separate them.

\begin{table}[t]
  \centering
  \small
  \caption{Weight $\times$ encoding factorial on DEV64, all cells at batch size 1. Entries are $\Dempty$ / $\Ddonor$ in nats per target token. ``Main'': plain writer at 100M; ``delta'': memory-aware fine-tune at 100M.}
  \label{tab:factorial}
  \begin{tabular}{llccccc}
    \toprule
    Weights & Encoding & $k=0$ & $k=1$ & $k=4$ & $k=8$ & $k=16$ \\
    \midrule
    main  & $\Tz$-only     & .0668 / .0234 & .0654 / .0179 & .0635 / .0168 & .0655 / .0163 & .0626 / .0127 \\
    main  & memory-aware   & .0626 / .0270 & .0627 / .0206 & .0648 / .0172 & .0613 / .0152 & .0647 / .0122 \\
    delta & $\Tz$-only     & .0560 / .0315 & .0555 / .0296 & .0550 / .0271 & .0516 / .0249 & .0538 / .0219 \\
    delta & memory-aware   & .0517 / .0320 & .0522 / .0290 & .0520 / .0252 & .0492 / .0232 & .0538 / .0217 \\
    \bottomrule
  \end{tabular}
\end{table}

\paragraph{Weight $\times$ encoding factorial.} \Cref{tab:factorial} crosses the two sets of weights with the two encodings at evaluation, all at the same batch layout. Under Bonferroni correction over 60 comparisons, 17 are significant:
\begin{itemize}[leftmargin=*,itemsep=2pt]
  \item \textbf{The shift lives in the weights.} With the encoding held fixed, changing the weights lowers $\Dempty$ significantly at all five distances under both encodings (by 0.008--0.014 with $\Tz$-only encoding) and raises $\Ddonor$ (by 0.008--0.012 with $\Tz$-only encoding, significant at $k\ge1$; with memory-aware encoding significant only at $k=16$). Under $\Tz$-only encoding, 70\% of $\Ddonor$ survives 16 distractors with the delta weights, against 54\% with the main weights.
  \item \textbf{Switching the encoding at evaluation brings no detected gain.} With the weights held fixed, memory-aware encoding lowers $\Dempty$ significantly only at $k=0$, and no corrected comparison shows a gain. No interaction term excludes zero. The layout control (main weights and $\Tz$-only encoding at batch 12 versus batch 1) differs by at most 0.0013 in point estimate.
\end{itemize}
The larger $\Ddonor$ must be read together with the absolute losses. At $k=0$, changing the weights moves self NLL from 1.7741 to 1.7849 and donor NLL from 1.7975 to 1.8164. Both get worse, and the gap grows because the donor gets worse faster. Training with memory-aware encoding therefore changed \emph{how} the writer writes: memory became more specific to its document and more robust to distractors, but less helpful overall on these news events. On DEV64 it did not make the writer better at predicting the document it wrote. Over the long seven-task chain, the same writer drifts less than the plain writer while answering slightly more (\cref{sec:longchain-drift}). This is a single training run, and the mechanism is unknown.

%% file: sections/07_longchain.tex
\section{Long Write Chains: Stability, Retention and Acquisition}
\label{sec:longchain}

The seven-task chain puts the writer far outside its training distribution. Training trajectories have at most 32 blocks of raw text. Here a single state receives 35,000 writes of short, templated examples from tasks the writer has never seen. This section reports what the chain does to stability, to what is stored, to what is acquired and to the base.

\subsection{Stability}
\label{sec:longchain-stability}

Both writers completed 35,000 consecutive writes on both stacks without numerical failure. Every state passed its integrity checks before and after each evaluation, and queries never modified the state. The delta writer, whose encoding reads its own memory, showed no runaway self-reinforcement over the chain: its final accuracy and drift are close to those of the plain writer (below). The two stacks write numerically different states from the same evidence. At the end of the chain $\norm{\Delta_{\mathrm{A800}}-\Delta_{\mathrm{H200}}}/\norm{\Delta_{\mathrm{H200}}}$ is 6.3\% for the plain writer and 4.2\% for the delta writer, while row sets, counts and clocks agree exactly. This is why accuracy (A800) and drift (H200) are reported as two trajectories.

\subsection{Acquisition and retention}
\label{sec:longchain-acc}

\begin{table}[t]
  \centering
  \small
  \caption{Seven-task accuracy (out of 64 per task) before the task's own examples are written $\to$ right after they are written $\to$ after the seventh task. The last column is the N-gram arm of~\citep{wu2026drift} right after learning each task.}
  \label{tab:seven-acc}
  \begin{tabular}{lccc}
    \toprule
    Task & Plain writer & Delta writer & \citep{wu2026drift} N-gram \\
    \midrule
    SME & 10 $\to$ 7 $\to$ 6   & 10 $\to$ 7 $\to$ 7   & 13 \\
    NAV & 6 $\to$ 5 $\to$ 5    & 6 $\to$ 6 $\to$ 6    & 6 \\
    HYE & 33 $\to$ 28 $\to$ 29 & 32 $\to$ 32 $\to$ 32 & 61 \\
    AFB & 5 $\to$ 2 $\to$ 2    & 6 $\to$ 2 $\to$ 2    & 42 \\
    SQI & 24 $\to$ 27 $\to$ 25 & 27 $\to$ 27 $\to$ 26 & 50 \\
    MKD & 29 $\to$ 27 $\to$ 28 & 26 $\to$ 27 $\to$ 26 & 54 \\
    AMH & 31 $\to$ 30          & 31 $\to$ 30          & 57 \\
    \midrule
    Total after task 7           & 125/448 (27.9\%) & 129/448 (28.8\%) & 255/448 (56.9\%) \\
    Net acquisition              & $-12/448$ & $-7/448$ & --- \\
    Change after learning (tasks 1--6) & $-0.26$pp & $-0.52$pp & $-7.29$pp \\
    \bottomrule
  \end{tabular}
\end{table}

\Cref{tab:seven-acc} separates acquisition from retention. Writing a task's own 5,000 examples does not raise its accuracy. Summed over the seven tasks, the change from writing is $-12/448$ for the plain writer and $-7/448$ for the delta writer (5 and 4 items turned correct, 17 and 11 turned wrong). Two tasks improve: SQI for the plain writer (24 to 27) and MKD for the delta writer (26 to 27). After the seventh task the chains hold 125/448 and 129/448, against 255--309/448 for the three arms of~\citep{wu2026drift}. The two writers do not differ significantly (paired McNemar $p=0.42$).

What is stored changes little afterwards. From the moment each task is written to the end of the chain, accuracy on the first six tasks changes by $-0.26$pp (plain) and $-0.52$pp (delta), against $-7.29$pp for the Adam-trained N-gram arm, $-11.72$pp for shared LoRA and $-7.03$pp for merged LoRA. Answer NLL tells the same story. Writing a task raises its answer NLL by 0.069 (plain) and 0.081 (delta) on average; afterwards, NLL on the first six tasks changes by $-0.001$ and $-0.019$, against $+0.305$ for the N-gram arm of~\citep{wu2026drift}.

The two halves must be read together. Little changes afterwards largely because little was acquired. These numbers show that later writes did not erode what the state held; they do not show that the method learns and then keeps. The positive evidence for storing and keeping new information is the document-specific memory of \cref{sec:memory}.

\subsection{Drift}
\label{sec:longchain-drift}

\begin{table}[t]
  \centering
  \small
  \caption{Small-$k_i$ drift (relative RMS) at the last prompt token, from right after each task is written to the end of the chain (plain writer, H200). Last column: N-gram arm of~\citep{wu2026drift} at the head input after task 7.}
  \label{tab:seven-ki}
  \begin{tabular}{lcccc}
    \toprule
    Task & Head input & Layer 48 & Layer 9 & \citep{wu2026drift} N-gram \\
    \midrule
    SME & 23.3\% & 42.8\% & 290\% & 40.4\% \\
    NAV & 21.3\% & 32.7\% & 211\% & 62.0\% \\
    HYE & 12.4\% & 21.1\% & 52\%  & 32.6\% \\
    AFB & 11.2\% & 21.3\% & 55\%  & 37.3\% \\
    SQI & 19.9\% & 31.7\% & 180\% & 36.8\% \\
    MKD & 19.2\% & 33.6\% & 265\% & 26.0\% \\
    \midrule
    Mean & 17.9\% & 30.6\% & --- & 39.2\% \\
    \bottomrule
  \end{tabular}
\end{table}

\paragraph{Small $k_i$.} At the head input, the learned tasks drift by 17.9\% on average from the end of their own writing to the end of the chain (\cref{tab:seven-ki}), against 39.2\% for the N-gram arm of~\citep{wu2026drift}. The delta writer gives 17.5\%. The comparison holds at matched rounds. For SME after one, two and three further tasks: 12.8\%, 18.7\% and 19.4\%, against 32.5\%, 40.1\% and 41.6\%. For NAV after one and two further tasks: 16.3\% and 17.2\%, against 50.8\% and 51.8\%. Inside the network the picture differs: at layer~9, four of the six tasks drift by more than 180\%. The small output-side drift sits on top of large early-layer changes.

\paragraph{Big $K$.} On general inputs the chain drifts more than the N-gram arm of~\citep{wu2026drift}. At the head input, math probes drift by 41.8\% and code probes by 23.8\% from the empty state (mean 32.8\%), against 26\%/15\% (mean 20.75\%) for that arm. This is still below merged LoRA (49.67\%) and shared LoRA (68.12\%). The two probe sets reach this level in different ways:
\begin{itemize}[leftmargin=*,itemsep=2pt]
  \item \textbf{Math} drift enters at the injection. It is 0 at layer~1, before the injection, and 198\% at layer~2 after the first task; it then decays with depth to 42.5\% at the head. It is fully present after the first task (5,000 writes) and does not accumulate: after the seventh task it is 188\% at layer~2 and 41.8\% at the head.
  \item \textbf{Code} drift is 1--3\% at layer~2 and builds up in later layers through attention, reaching 24\% at the head.
\end{itemize}
The math pattern is what a few hot rows read directly at the probe position would produce, and \cref{sec:rows} confirms it. Big-$K$ drift is not a behavioural loss here: math answer NLL falls from 0.904 (empty) to 0.846, and code NLL is unchanged (0.247 to 0.246). The rows that move the math representations also help predict the math answers (\cref{sec:rows-zero}). We nevertheless count 32.8\% as a defect. Drift on general inputs is a budget the method must stay within, not a cost to be offset by gains.

\begin{table}[t]
  \centering
  \small
  \caption{Plain versus delta writer on the seven-task chain, after the seventh task unless stated. Drift is relative RMS (H200); accuracy and recall are out of 448 (A800).}
  \label{tab:writers}
  \begin{tabular}{lcc}
    \toprule
    & Plain writer & Delta writer \\
    \midrule
    Math big $K$, layer 2 & 188.0\% & 175.7\% \\
    Math big $K$, layer 2, after task 4 & 153.0\% & 94.8\% \\
    Math big $K$, layer 9 & 135.0\% & 125.7\% \\
    Math big $K$, layer 48 & 51.2\% & 49.3\% \\
    Head big $K$, math / code & 41.8\% / 23.8\% & 40.2\% / 22.5\% \\
    Small $k_i$ at the head, mean of six tasks & 17.9\% & 17.5\% \\
    Old-task answer NLL, change after learning & $-0.001$ & $-0.019$ \\
    Net acquisition & $-12$ & $-7$ \\
    Accuracy after task 7 & 125 & 129 \\
    Recall of written examples & 103 & 112 \\
    \bottomrule
  \end{tabular}
\end{table}

\paragraph{Plain versus delta writer.} The delta writer, which reads its memory while encoding, drifts less than the plain writer over the chain (\cref{tab:writers}). On the math probes its drift is smaller at every reported depth at the end of the chain. The gap is largest in early layers: after the fourth task, layer-2 drift is 94.8\% against 153.0\%, and layer-9 drift is 69.4\% against 108.8\%. At the head input its big $K$ is smaller at six of seven endpoints on math and at all seven on code. On learned tasks its small-$k_i$ drift is lower for four of six tasks, and their answer NLL improves by 0.019 after learning, against 0.001. The lower drift costs nothing in accuracy: net acquisition is $-7$ against $-12$, 129 against 125 test items are answered after the seventh task, and 112 against 103 written examples are recalled. \Cref{sec:rows-zero} repeats the comparison after zeroing the shared rows.

\subsection{Recall of written examples}
\label{sec:longchain-recall}

Weak acquisition could mean the state stores the examples but does not generalise from them. We test this directly. From the final states, both writers answer the same 448 support examples that were written into memory (64 per task), using their original prompts. The plain writer answers 103/448 (22.99\%) and the delta writer 112/448 (25.00\%). The difference is +9 items (paired McNemar $p=0.049$, uncorrected, exploratory). Neither state reproduces a large share of what it was given, so the data do not support ``remembered but not generalised''. The writes did not bind each answer to an address from which the answer can be read back (\cref{sec:rows-entry}).

\subsection{LoCoMo}
\label{sec:longchain-locomo}

\begin{table}[t]
  \centering
  \small
  \caption{LoCoMo, categories 1--4 (1,540 questions), delta writer. F1 on 0--100; reference-answer NLL in nats/token, averaged per question. Intervals: 10,000 dialogue-cluster bootstrap resamples.}
  \label{tab:locomo}
  \begin{tabular}{lcc}
    \toprule
    Condition & F1 & Answer NLL \\
    \midrule
    self (own dialogue in memory)   & 0.44 & 4.034 \\
    empty                           & 0.44 & 4.018 \\
    donor (other dialogue in memory)& 0.44 & 4.049 \\
    full context                    & 56.79 [55.33, 58.34] & --- \\
    \midrule
    self $-$ empty                  & 0 & $+0.016$ [$-0.006$, $+0.035$] \\
    self $-$ donor                  & 0 & $-0.015$ [$-0.025$, $-0.004$] \\
    \bottomrule
  \end{tabular}
\end{table}

Writing the ten dialogues takes 426 block writes (215,278 tokens), about 12 seconds per 23,000-token dialogue on A800. With only the question in context, all three memory conditions answer every question with the same refusal, ``No information available''. No greedy output differs between self, empty and donor, and F1 is 0.44 in each. The same backbone with the whole dialogue in context reaches 56.79 (\cref{tab:locomo}).

Reference-answer NLL separates the sources of memory without showing a benefit. Self is lower than donor ($-0.015$, interval below zero), so what was written depends on the dialogue. Relative to empty memory, the self difference crosses zero under both per-question and per-token averaging, and its sign depends on the averaging. Split by position, both self and donor make the first answer token worse than empty memory ($+0.13$ and $+0.17$). Given the correct prefix, both improve the remaining tokens slightly ($-0.020$ and $-0.016$), so this improvement is not specific to the dialogue that was written. These NLLs use a single-prefill path that differed from the token-by-token path by up to 0.34 nats in a pre-run check, so they are specific to that path. \Cref{sec:rows-entry} offers one explanation for the first-token result: the rows read where the answer starts were never written. This is consistent with the result but does not by itself establish its cause.

Zero-shot transfer from raw-text continuation to dialogue question answering therefore fails for this writer. Using written content at a different position from where it was written is the main capability the method lacks.

%% file: sections/08_rows.tex
\section{Locating Drift at the Row Level}
\label{sec:rows}

In~\citep{wu2026drift} drift could be measured but not attributed: Adam updates had spread over many rows of a table that no longer stored its pretrained values. Here every change is a patch on a known row, the addresses are deterministic hashes of tokens, and the plain writer is row-separable (\cref{prop:separable}). This makes it possible to say which rows a probe reads, how often each was written, how far it moved, and what removing it does.

\subsection{Which rows the math probes read}
\label{sec:rows-audit}

Using the native hash and tokenizer, we reconstruct the 16 physical rows read at the last position of every math probe and look them up in the saved states; this is a CPU audit with no model forward. All 128 math prompts end with the tokens \texttt{Solution}, \texttt{:} and a newline. Of the 16 rows read there:
\begin{itemize}[leftmargin=*,itemsep=2pt]
  \item the eight trigram rows for (\texttt{Solution}, \texttt{:}, newline) were never written: count and patch are zero;
  \item the eight bigram rows for (\texttt{:}, newline) are the rows of \texttt{Input:} followed by a newline, which occurs in every support example of the seven tasks. Each example writes each of these rows exactly once, with no further hash collisions. After the first task their counts are 5,000, after the seventh 35,000.
\end{itemize}
Their patches are several times larger than the pretrained rows: $\norm{\Delta_j}/\norm{\Tz_j}$ ranges over 4.48--10.75 after the first task and 4.41--10.03 after the seventh for the plain writer, and 4.61--11.74 after the first task for the delta writer. The level is reached within the first task. From the first to the seventh endpoint the concatenated patch of the eight rows changes by 8.8\% (plain writer) and keeps its direction (cosine 0.997). This matches the layer profile of \cref{sec:longchain-drift}: math drift appears at the injection after the first task and does not accumulate. Boundedness (\cref{prop:bounded}) holds throughout; what is missing is anything that ties the size of a patch to $\norm{\Tz_j}$.

\subsection{Zeroing the eight shared rows}
\label{sec:rows-zero}

\begin{figure}[t]
  \centering
  \includegraphics[width=\linewidth]{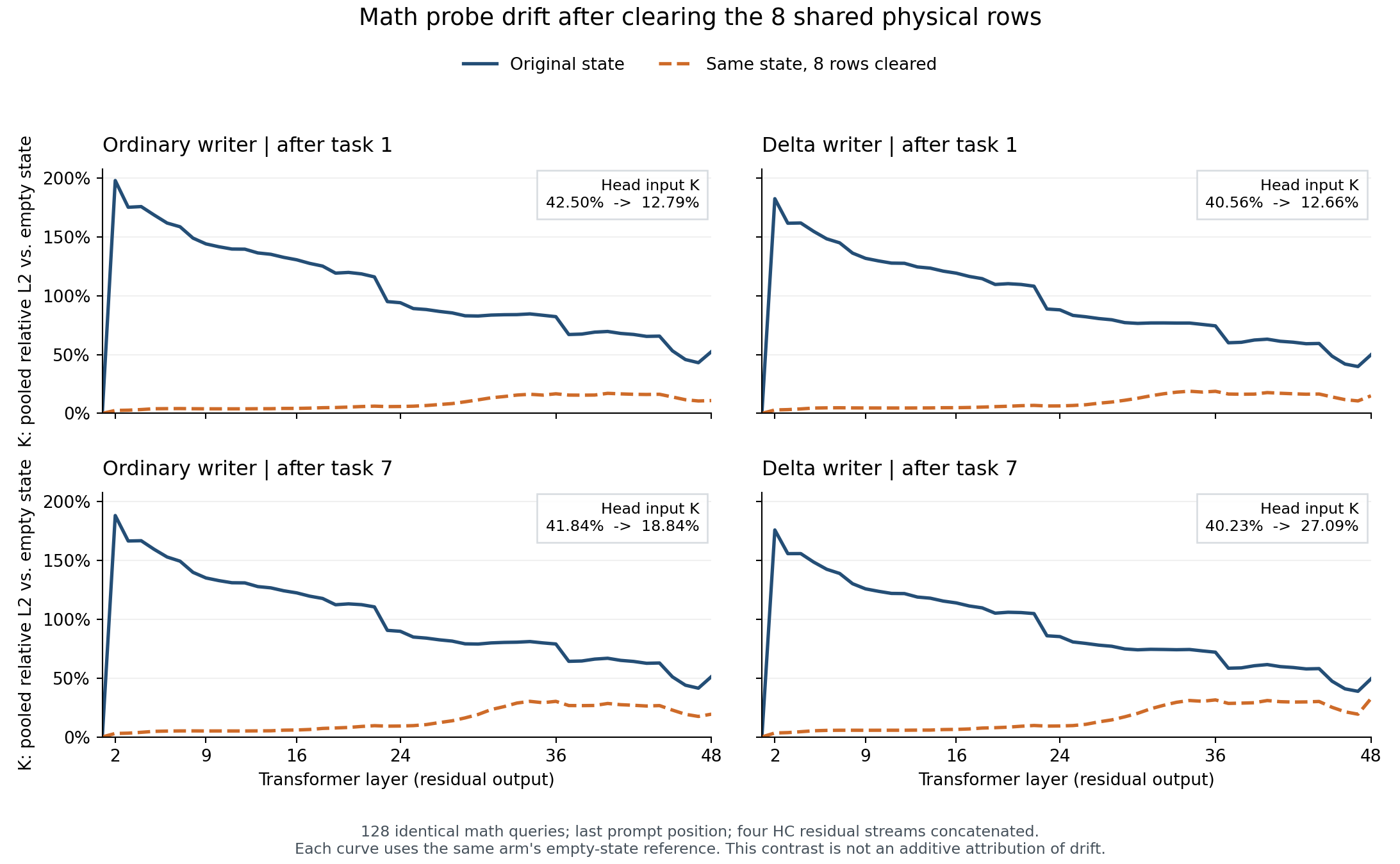}
  \caption{Big-$K$ drift of the 128 math probes at every layer, before (solid) and after (dashed) zeroing the patches of the eight shared bigram rows in the same state. Each curve is measured against the same writer's empty state.}
  \label{fig:math8}
\end{figure}

\begin{table}[t]
  \centering
  \small
  \caption{Zeroing the patches of the eight shared bigram rows. Big-$K$ drift of the math probes (relative RMS against the empty state) and their reference-answer NLL. Empty-state math NLL is 0.9039 for both writers.}
  \label{tab:math8}
  \begin{tabular}{lcccc}
    \toprule
    State & Layer 2 & Layer 9 & Head input & Math NLL \\
    \midrule
    Plain, task 1 & 197.7\% $\to$ 2.5\% & 143.9\% $\to$ 3.9\% & 42.5\% $\to$ 12.8\% & 0.841 $\to$ 0.901 \\
    Plain, task 7 & 188.0\% $\to$ 2.9\% & 135.0\% $\to$ 5.1\% & 41.8\% $\to$ 18.8\% & 0.846 $\to$ 0.902 \\
    Delta, task 1 & 182.3\% $\to$ 2.9\% & 131.7\% $\to$ 4.6\% & 40.6\% $\to$ 12.7\% & 0.849 $\to$ 0.902 \\
    Delta, task 7 & 175.7\% $\to$ 3.3\% & 125.7\% $\to$ 5.7\% & 40.2\% $\to$ 27.1\% & 0.855 $\to$ 0.901 \\
    \bottomrule
  \end{tabular}
\end{table}

We set the patches of these eight rows to zero in a copy of each state and keep $\Tz$, every other row and all other fields. Then we rerun the 128 math probes at all seven endpoints of both writers. For the plain writer, \cref{prop:separable} makes this the state the chain would have reached had these rows never been committed. For the delta writer it is an intervention on the final state only.

\Cref{fig:math8} and \cref{tab:math8} show the result. Layer-2 drift falls by 98.1--98.7\%, from 176--198\% to 2.5--3.3\%. Layer-9 drift falls from 126--144\% to 4--6\%. These eight rows are the source of the shallow drift on the math probes. Drift that remains at the head (12.7--27.1\%) builds up in later layers from other rows and from propagation through attention, as on the code probes. The intervention also compares the two writers. Zeroing the eight rows lowers the delta writer's final head drift from 40.2\% to 27.1\%, and the plain writer's from 41.8\% to 18.8\%, although the delta writer's patches on these rows are no smaller (4.6--11.7 times $\norm{\Tz_j}$ after the first task). Drift norms do not add across rows, so these numbers do not apportion the drift between the eight rows and the rest. For the delta writer the zeroed state is an intervention on its final state, not a state its chain would have reached (\cref{sec:analysis-sep}).

The same rows carry a gain. Zeroing them raises math answer NLL in all 14 writer $\times$ endpoint conditions, by 0.026--0.060, to 0.901--0.902, close to the empty-memory level of 0.904. Shared rows are therefore not simply contamination. With the current write rule, the generic gain on math inputs and the drift on them come from the same few patches, written by format tokens of unrelated tasks. Zeroing these patches removes both. Whether a nonzero cap can bound the drift and keep part of the gain has not been tested. This is the trade-off that a per-row budget must manage. Rows whose patches exceed their pretrained norms by a set factor are the ones a norm cap $\norm{\Delta_j}\le c\,\norm{\Tz_j}$ would bound. Row separability gives the effect of capping the final state directly; a cap enforced at every write requires replaying the write chains of those rows (\cref{sec:analysis-sep,sec:discussion}).

\subsection{The answer entry}
\label{sec:rows-entry}

The same audit explains part of the failure to use memory at a different position. In both the seven-task prompts and the LoCoMo prompts, the first answer token is predicted at the end of the chat template, after \texttt{</think>} and two newlines. All 448 recall prompts of \cref{sec:longchain-recall} and all 1,986 LoCoMo questions read the same 16 physical rows there.
\begin{itemize}[leftmargin=*,itemsep=2pt]
  \item In the seven-task chain the support examples use the same template, so these 16 rows are written throughout the chain. Answers to different lemmas and features meet in the same rows, where later writes can mix with or overwrite earlier ones.
  \item In LoCoMo, the dialogue text written into memory contains no template, so in all ten dialogue states these 16 rows were never written: their patches and counts are zero. Across the 426 blocks and 3.4 million address events of the dialogues, no write lands on them.
\end{itemize}
Memory can still reach the first answer token, but only indirectly: patches read at earlier prompt positions must propagate through attention and the convolution. The writer was trained on raw-text continuation, where the useful pattern is to write at a position and read at the same n-gram later. It was never trained to make content written in one context readable from an answer position whose own rows carry nothing. A plausible explanation for both the weak recall of \cref{sec:longchain-recall} and the LoCoMo result is therefore this: the writer learned to write and read at the same position, while question answering needs retrieval across positions. We regard this as the main open problem of the method rather than a property of the addressing alone. The writer's features include right context and a block-level pool, so it can see the answer when it writes earlier rows, and queries can attend to rows read earlier in the prompt.

%% file: sections/09_discussion.tex
\section{Discussion}
\label{sec:discussion}

\paragraph{Where memory should live.} The native N-gram table is sparse and addressable, and the model already reads it. Writing there needs no new memory module, no change to the backbone and no external store. If future models ship with such tables, the table may be the natural place for continual learning: a writable store of knowledge inside the model that leaves the backbone untouched.

\paragraph{Isolation versus binding.} Exact-token addresses are what make the method's retention and attribution possible. A write touches only the rows its n-grams address, and an input whose prefix reads none of them is unchanged. The same property limits use. A row can only help when the same n-gram is read again, so content written at one position is not directly available at another; \cref{sec:rows-entry} shows the answer position reading rows that were never written. Content addressing would give binding, but writes would overlap again, and with overlap comes interference. This is the central trade-off of the approach. Our results suggest that what is missing is a writer trained to make written content usable across positions, and possibly a memory module with its own attention for reading it (see the last paragraph of this section). The next step is an episodic objective, in which evidence is written and then queried at other positions, with distillation from the same frozen model given the evidence in context. That distillation is the component whose removal costs MemT the most~\citep{memt2026}. A CPU check on existing LoCoMo losses can decide first whether the n-grams of the answers are reachable at all, or whether a small content-addressed memory is needed beside the table.

\paragraph{Drift as a budget.} Pretrained knowledge is the most valuable thing the model has, so absolute drift on general inputs must be bounded and not traded against gains. The method already makes drift locatable: we know which rows moved, how far and through which tokens. The default should therefore include a per-row cap,
\begin{equation}
  \Delta_j\leftarrow\Delta_j\cdot\min\!\Big(1,\ \frac{c\,\norm{\Tz_j}}{\norm{\Delta_j}+\varepsilon}\Big),
\end{equation}
with $c$ fixed in advance by the drift budget for big $K$ and small $k_i$, not tuned for gains. Since the read-out is injected at layer~2, bounding each patch relative to its pretrained row bounds the perturbation at the injection directly. \Cref{sec:rows-zero} shows the cost: the hot rows also carry the generic gain, and zeroing them removes it; how much of it a nonzero cap keeps is untested. Recovering gains inside the budget is the role of \emph{multi-slot rows}. The few most-written rows would get several slots keyed by context, so that a generic signal and specific content are stored apart and used only where the context matches; all other rows keep a single value. A larger hash table does not help, because it removes collisions but not the sharing of one n-gram by many contexts. For the plain writer, the effect of locking any set of rows can be computed from existing final states. A cap enforced during writing needs only the write chains of the capped rows to be replayed.

\paragraph{Encoding with or without memory.} Encoding evidence with the $\Tz$-only model is a choice, not a necessity. It anchors targets to the evidence, keeps rows separable, makes training and use identical, and lets all blocks be encoded in one batch, which on Qwen3.8 is the difference between about 3,200 and 940 target tokens per second. The cost is that the writer cannot interpret new content in the light of what it has stored, for instance resolving a pronoun to an entity written earlier. Memory-aware encoding removes that limitation. It serialises writing and makes the state sensitive to small numerical differences, though in our measurements these stayed on a bounded plateau below read-side noise (\cref{sec:memory-encoding}). The factorial shows that training with it changes how the writer writes. Over 35,000 writes, the resulting writer drifts less than the plain writer, most clearly in early layers, with slightly higher accuracy and recall (\cref{sec:longchain-drift}). On news continuation it trades total gain for document specificity. The two writers therefore remain two lines of the method. The plain writer gives batch training, row-level attribution and reproducibility. The delta writer is the route to integrating new evidence with what is already stored. What it still lacks is a training signal that rewards such integration: evidence spread over several related documents, with corrections and revisions, rather than continuation of one document at a time.

\paragraph{Who decides what is written.} Every token of the evidence produces write events; how much each writes is decided by the learned gate, not by hand-written masks. Raw-text training teaches the writer what helps later continuation. It does not teach what a later question will need, and it does not teach restraint on format rows that every example shares. The one boundary still set by hand is what counts as evidence. User documents, user corrections and verified outcomes count. The model's own generations are not written back by default, because a persistent write that conditions all of a user's later requests is where self-reinforcement would do damage.

\paragraph{Memory that belongs to the user.} Each user's memory is a separable set of sparse rows, so it can be kept on the client and supplied per request, while the backbone and writer are shared. A dialogue of 13,000--25,000 tokens gives 78--145 MB of FP32 rows before any compression. A client can only change its own state; merging the patches of many users into a shared one is where privacy and poisoning questions arise. A portable state must carry its compatibility information: the backbone, tokenizer and writer versions, the addressing scheme, the write operator and decay, and a schema and layout version with a content checksum. A resumable full state, with counts and clock, must be distinguished from a read-only snapshot for inference.

\paragraph{Fast writing, slow consolidation.} The method is the fast half of a two-speed system. In use, it writes sparse rows quickly. During a ``sleep'' phase, signals that have proved stable and general, such as the generic gain now carried by hot rows, would be consolidated into a protected shared part, and specific content would stay in sparse rows. This is the division of labour between hippocampus and neocortex in complementary learning systems~\citep{mcclelland1995cls}, and the starting point of our 2023 work on online LoRA updates with sleep cycles. A forward write path is also what a self-improving system needs: the model writes its experience into its own memory and works with that memory afterwards. Connecting the two is the subject of our next paper.

\paragraph{Limitations.} Two experiments show that the method does not yet work on long-chain tasks that need written content at a different position. On the seven-task chain, net acquisition is negative and only 23--25\% of written examples can be recalled. On LoCoMo, every memory condition answers with the same refusal. Our writers went through only the first half of MemT-1.0's recipe~\citep{memt2026}: continuation on raw text (100M tokens), with no post-training. MemT follows its 20B tokens of continued pretraining with 14B tokens of post-training on memory-specific data, distilled from the frozen backbone given the full history. That stage adds about 10.69 points in its ablation, and it targets exactly the cross-position use our writers lack. Hot rows have no budget relative to their pretrained norms, and general-input drift is 1.6 times that of the N-gram arm of~\citep{wu2026drift}. Budget limited us to one model with a native N-gram table. We did not train writers for LongCat-Flash-Lite, the model edited by EngramEdit~\citep{engramedit2026}, which adds its n-gram embeddings at the input without a gate. Nor did we train one for DeepSeek-V4.1-Flash-1017, which injects two tables at different layers. A writer is tied to its backbone, so each would need its own; the second would also test the drift bound with more than one injection point. All results come from one training seed.

\paragraph{Why the writer learns only same-position continuation.} The most likely reason, still to be tested, is the training objective. Raw-text continuation rewards writing, at the row of an n-gram $k$, what follows $k$, and reading it back at the same position when $k$ next appears. DEV64 continuation uses memory in exactly this way, which is why it works there. Question answering instead needs information written at the positions of the question's content to be moved by attention to the answer position, and the writer was never asked to write for that. Two observations point the same way. The 0.6B prototype trained with a question-answering objective answers 16 of 16 facts after one write, against 2 of 16 with empty memory (\cref{sec:memory-proto}); after the switch to raw text this ability was not kept. And in \citep{wu2026drift}, gradient updates of the same table reached 56.9\% on the seven tasks through the same frozen attention. So the table, read through frozen attention, can carry the answers when it is trained with answer gradients. Whether a forward writer trained for cross-position use can reach that, or whether reading across positions also needs a memory module with its own attention, as in MemT-1.0~\citep{memt2026}, remains open.

%% file: sections/10_conclusion.tex
\section{Conclusion}
\label{sec:conclusion}

Delta N-gram writes into a language model's own N-gram table with forward passes only. It needs no gradients, labels or optimizer state at use time and never touches the pretrained table. A 4M-parameter writer trained on raw text with the pretraining objective learns document-specific memory that survives later writes. A second writer, which reads its own memory while encoding evidence, writes memory that is more specific to its document and drifts less over long chains. Both stay stable over chains of 35,000 writes. Because every change is a patch on a known row, drift can be traced to and intervened on at individual rows, which is what turns a drift budget from a measurement into something that can be enforced. Two problems remain. First, written content is not yet usable at other positions; one explanation still to be tested is that the answer position reads rows the evidence never wrote. Second, a few format rows shared by every example carry both the generic gain and the drift on general inputs. A cross-position training objective, a per-row cap set by the drift budget, and context-keyed slots for hot rows address these two problems; they are the next steps.

%% file: sections/A_appendix.tex
\section{Native Implementation and Reproducibility}
\label{app:engineering}

\subsection{Training on Qwen3.8}
\label{app:training}

\paragraph{Layout.} Training uses tensor and pipeline parallelism 1, expert parallelism 4 and data parallelism 4 on four H200 GPUs, with three trajectories per GPU and twelve per update. Gradients are summed across GPUs and normalised once by the number of target tokens in the whole group. The pretrained table holds about 51 billion parameters and is kept in CPU memory.

\paragraph{Four-stage schedule.} With $\Tz$-only encoding, $\hat H$ does not depend on the memory, which allows a staged schedule:
\begin{enumerate}[leftmargin=*,itemsep=1pt]
  \item encode all blocks of all trajectories in one batch;
  \item run the D1 write chain block by block;
  \item run all queries forward and backward in one batch, giving the gradient of each block's patch;
  \item backpropagate along the write chain into the writer.
\end{enumerate}
This reaches 3,000--3,300 target tokens per second, against about 940 for the interleaved schedule. Its loss agrees with the sequential schedule to within 0.001. Gradient differences are at the level of the routing-noise floor below. Memory-aware encoding requires the sequential schedule.

\paragraph{Small patches in a BF16 model.} The backbone runs in BF16, and a patch of order $0.01\cdot\mathrm{RMS}(\Tz)$ added to $\Tz$ would mostly be rounded away. The read-out is therefore split into a BF16 projection of the base table and an FP32 projection of the patch, which are added before the injection. The empty patch read through this split branch is the baseline for every comparison; it is not bitwise equal to the original BF16 read-out.

\paragraph{Routing-noise floor.} With ten of 512 experts active per token, a patch that differs only in its last bits can switch the experts of some tokens. Between such near-identical states the writer's gradient changes by about 31.6\% relative. Schedule equivalence is judged against this floor, not bit equality.

\paragraph{Training and evaluation agree.} Two fixes made forward passes with and without gradient tracking bit-identical: fixing the router's top-$k$ to sorted order in both modes, and running two fused mixing functions eagerly. Evaluation runs without gradient tracking. With tracking on but no backward pass, the graph accumulated and memory grew from 68 to 143\,GB until evaluation stalled; without it, memory stays near 78\,GB.

\subsection{Inference with live writes}
\label{app:inference}

For generation we extend vLLM's Qwen3.8 read-out with the same FP32 patch branch. Normalisation, gating and convolution of the read-out run in FP32, and the result is cast to BF16 before it joins the residual stream; the patch is never merged into the BF16 table. Under tensor parallelism 8, one rank computes each D1 update and all ranks adopt its result. The state is updated only between requests. Every request checks the state hash before and after, and answering never modifies the state. Writing a 23,000-token dialogue takes about 12 seconds on A800 GPUs. The state of a 13,000--25,000-token dialogue is 78--145\,MB in FP32 without compression.

\subsection{Reproducibility}
\label{app:repro}

Every run is bound to the SHA-256 hashes of its source code, checkpoint and data. Exact resumption of training is allowed only when the code manifest hash matches; a code change goes through a registered migration that leaves the old sources untouched. Generations are sealed before they are scored, and key numbers were recomputed independently from raw per-item records. The DEV64 donor pairs and distractors are fixed by hash before training. The code is available at \url{https://github.com/af-74413592/delta-ngram}.